\documentclass[review]{elsarticle}

\usepackage{lineno,hyperref}
\modulolinenumbers[5]
\usepackage{amsthm}
\usepackage{amsmath}
\usepackage{amssymb}
\usepackage{subfig}
\usepackage{xcolor}

\journal{}

\begin{document}
	
	\begin{frontmatter}
		
		\title{A Framework for 3D Tracking of Frontal Dynamic Objects in Autonomous Cars}

		\author{Faraz Lotfi\tnoteref{Corresponding author}, Hamid D. Taghirad}
		\address{Advanced Robotics and Automated Systems (ARAS), Industrial Control Center of Excellence, Faculty of Electrical Engineering, K. N. Toosi University of Technology, Tehran 1969764499, Iran}
		\tnotetext[Corresponding author]{Corresponding author.\\E-mail addresses: F.lotfi@email.kntu.ac.ir (F. Lotfi), Taghirad@kntu.ac.ir (H.D. Taghirad).}
		
		
		
		

		\begin{abstract}
			Both recognition and 3D tracking of frontal dynamic objects are
			crucial problems in an autonomous vehicle, while depth estimation as
			an essential issue becomes a challenging problem using a monocular
			camera. Since both camera and objects are moving, the issue can be
			formed as a structure from motion (SFM) problem. In this paper, to
			elicit features from an image, the YOLOv3 approach is utilized beside an OpenCV tracker. Subsequently,
			to obtain the lateral and longitudinal distances, a nonlinear SFM
			model is considered alongside a state-dependent Riccati equation
			(SDRE) filter and a newly developed observation model. Additionally,
			a switching method in the form of switching estimation error
			covariance is proposed to enhance the robust performance of the SDRE
			filter. The stability analysis of the presented filter is conducted
			on a class of discrete nonlinear systems. Furthermore, the ultimate bound of
			estimation error caused by model uncertainties is analytically
			obtained to investigate the switching significance. Simulations are
			reported to validate the performance of the switched SDRE filter.
			Finally, real-time experiments are performed through a multi-thread
			framework implemented on a Jetson TX2 board, while radar data is
			used for the evaluation.
		\end{abstract}
		
		
	\end{frontmatter}
	
	
	\section{Introduction}
	3D tracking of frontal dynamic objects is highly
	lucrative to obtain rich information about the surrounding
	environment in an autonomous vehicle. Since the movement is in two
	directions, the main purpose is
	to obtain the lateral and longitudinal distances of the frontal
	dynamic objects while recognizing them, simultaneously. The depth,
	however, is not directly measurable by a single camera, due to the
	3D projection, and therefore, it has to be estimated. Deep learning
	has been proved to be applicable to solve this issue. As reported
	in~\cite{ref4}, a method based on deep fully convolutional residual
	networks may be used to formulate depth estimation as a pixel-wise
	classification task and the main concept is to discretize the
	continuous ground-truth depth into several bins and label them regarding their depth ranges. In~\cite{ref5}, a deep convolutional
	neural field model is presented which may be used for depth
	estimation of general scenes without any extra information such as
	geometric priors. Furthermore, in~\cite{ref6} and~\cite{ref7}, an
	unsupervised training approach is proposed that enables a CNN to
	estimate the depth, despite the absence of ground truth depth data.
	\textcolor{black}{Even though these methods promise a depth map of both indoor and
		outdoor environments, their computational cost prevents to simply
		utilize them in real-time applications. Moreover, these approaches are appropriate for static scenes or when a dense depth map is needed \cite{ref10_6}. In the application mentioned in this article, frontal dynamic objects are counted as a sparse depth map requirement with lateral distance estimation. To the best knowledge of the authors these two objectives are not addressed simultaneously in any of the presented approaches in the literature. }
	
	On the other hand, there are other approaches reported in the
	literature, which are based on obtaining structure from motion. In
	these methods, movements are used to estimate the depth ~\cite{ref8,
		ref9}. \textcolor{black}{Basically, an SFM problem can be solved via deep learning techniques or classic methods. As the state-of-the-art deep learning based solution, \cite{deepsfm1} may be mentioned which proposes a physical driven architecture regardless of the accurate camera poses. This contains two cost volume based architectures for depth and pose estimation. The second representative approach is \cite{deepsfm2} which is based on robust image retrieval for coarse camera pose estimation and robust local features for accurate pose refinement. Deep learning based methods such as this approach solve the SFM problem through obtaining a dense depth map or a dense 3D map while the focus of this paper is on estimating of the longitudinal and lateral distances to specific frontal objects, which may be seen as producing a sparse map. Classic approaches, however, are more relevant taking an observer along with
		a model that relates the observations of feature points in a 2D
		image to the corresponding Euclidean coordinates in the inertia
		frame~\cite{ref10, ref11, ref12}. To be more specific, consider
		a moving object Euclidean position is needed and a single moving
		camera is utilized to obtain it. The first step is to determine
		feature points of the object and to take their positions as
		observations in the image. Then, a motion model together with an
		observer can be exploited to estimate both the longitudinal and
		lateral distances of the feature points. 
		Regarding the second step, methods
		such as particle filter~\cite{ref19}, variable structure
		\cite{ref20}, high-order sliding-mode~\cite{ref21}, extended Kalman
		Filter (EKF)~\cite{ref22} and methods based on SDRE~\cite{ref23} may
		be considered.} 
	
	To specify the feature points automatically, deep CNNs may be applied to detect and
	recognize the objects ~\cite{ref13, ref14,ref13_2}. There are diverse
	architectures that provide the capability to recognize visual
	patterns directly from pixel images with negligible preprocessing.
	For instance, AlexNet, ZFNet, ResNet, VGGNet, and YOLO are introduced
	with a different number of convolutional, pooling and fully
	connected layers for various applications \cite{ref15}. Furthermore,
	OpenCV trackers \cite{ref32} such as tracking-learning-detection
	(TLD)~\cite{ref16} and Median-flow~\cite{ref17} may be implemented
	alongside CNNs to reduce the computational cost.
	
	\section{Related Works and Motivation }
	\textcolor{black}{To further clarify the gap in the literature and how this paprer has addressed it, consider two main components in solving an SFM problem in a real-world application. For the first component, consider a classic nonlinear observer that is employed to estimate the state variables while using the environmental observations. To produce the observations, however, it is necessary to take a strong image processing algorithm as the second component when utilizing a camera as the sensor. In the first component, several methods are proposed among which the unknown input observer (UIO) may be considered as the state-of-the-art (SOTA) since it is widely used in the estimation problems where unknown inputs such as disturbances are present \cite{UIO1,ref10_3,ref10_5}. In an SFM problem, the intended object velocities are taken as unknown inputs which makes the UIO a beneficial approach \cite{ref10,ref10_1}. This observer basically eliminates the unknown input effects and is commonly used in the literature for various applications \cite{ref10_2,ref10_4}. The main assumption in employing a UIO approach is, however, to have a perfect model of the system which indeed is a limiting hypothesis. Thus, the first contribution of this paper is to propose a new discrete-time nonlinear observer to address the SFM problem in presence of model uncertainties. A comparison study is performed with respect to the UIO to reveal the capability of the presented estimator. According to the reported results, the UIO performs better when the model is perfect; however, the proposed nonlinear observer outperforms the UIO in the presence of modeling uncertainty. Besides, the SFM problem is not applicable to the autonomous vehicle field since it needs the movement of the camera in all the three directions based on some basic assumptions \cite{ref10,ref11}. In this regard, a new observation model is necessary to be added to the SFM problem to further generalize this approach to the autonomous vehicle area. Giving attention to the second component, almost all of the presented approaches employ a simple method to produce the observations such as utilizing a pre-mounted marker on the target or considering specific target features while focusing on the first component of the problem \cite{ref11,ref10_1,ref10_2}. This, of course, is not applicable to the problem of estimating the distances to the frontal objects about which there is no a priori-knowledge. Consequently, a real-time image processing method is necessary to produce the observations robustly. Finally, all the previous approaches are presented to estimate the variables for short distances. This research, however, covers estimation of the lateral and longitudinal distances up to $80$ meters away from the host vehicle which is crucial in case of an autonomous vehicle application. \\
		The SFM problem addressed in this paper includes the estimation of all the coordinates corresponding to a feature point. However, in the machine learning field, especially in the deep learning methods, the simplified problem of just estimating the depth is taken into consideration. In \cite{ref_new_1}, it has been claimed that the first end-to-end learning-based model is developed to directly predict distances for
		given objects in the images. The mentioned research is associated with only the depth estimation, in the absense of lateral distance prediction. Furthermore, as the stringent requirement for this CNN model the object bounding boxes are needed, which in turns requires adding another deep CNN like YOLO, and thus downgrading its performance. The other similar work is presented in \cite{ref_new_2}, where an RCNN \cite{ref_new_3} based architecture is used as the main model along with another shallow network to accomplish the task of the detection and distance estimation, simultaneously. Comparing this approach with the one presented in this paper, obviously, the RCNN architecture yields poor performance in terms of computation cost and execution speed, while the lateral distance is not considered at all. Finally, as a powerful approach, \cite{ref_new_4} proposes DisNet to estimate the depth up to $100$ meters while its accuracy reaches to $90\%$. It uses a deep neural network consisting of three hidden layers with $100$ neurons at each layer, while it is just evaluated on a number of specific static scenes, and not considering the lateral distance like other mentioned approaches. As a result, it can be remarked that this is the first time that the problem of both the lateral and longitudinal distances estimation is formulated and worked out for dynamic environment. }   
	
	In this paper, a rigid body kinematic motion model as exploited in
	\cite{ref11} is used alongside YOLOv3~\cite{ref18}, Median-flow
	tracker and a newly designed filter to estimate the longitudinal and
	the lateral distances of frontal objects. Since the dynamics of these two variables are not
	linear, nonlinear filters are preferable. 
	\textcolor{black}{SDRE filter high degree of freedom provides
		singularity avoidance making it a proper choice employed in various applications as the optimal solution for nonlinear estimation and control problems \cite{SDRE_1,SDRE_2,SDRE_3}. This filter incorporates the nonlinearity into the estimation equations based on a parameterization which compared to the EKF is counted as a significant advantage since the EKF linearizes the nonlinearity \cite{ref_SDRE}. However, model uncertainty yields limitations
		in the SDRE filter implementation which can be further empowered based on a switching concept presented in this research.}
	Consequently,
	a discrete time-dependent switching-based SDRE filter is proposed,
	to further improve the structural robustness of the common SDRE
	filter against model uncertainty. More accurately, incorporation of
	switching concepts and traditional SDRE filter increases the SDRE
	filter performance in the presence of uncertainties. It is worth
	noting that in this approach, the switching of estimation error
	covariance matrix with a predetermined frequency is considered in
	the filter dynamics. The main outcome of the proposed method is
	regulation of estimation error covariance matrix eigenvalues which
	result in more robustness and a lower amount of ultimate error bound.
	Note that, the concept of using switching is presented in the author's previous
	research~\cite{ref_khodam2} for continuous-time systems. In this
	paper, the proposed estimator is extended to the discrete-time
	systems, while it is applied to the application of autonomous cars.
	\textcolor{black}{The literature regarding the stability analysis and the equations of the estimator is far different of that of the continous-time. The theorems along with the lemmas derived in this article can be used in further theoretical developments in discrete-time related approaches.}
	Considering the application of the proposed filter in estimation
	theory, stability analysis based on the Lyapunov theorem is
	performed on a class of nonlinear systems. Then robustness analysis
	is accomplished and the influence of switching on the ultimate bound
	of estimation error is investigated. Simulations are given to reveal
	the effectiveness of the proposed filter, and finally, real-time
	implementation results of the filter are reported to verify the
	theoretical developments in the application of autonomous
	vehicles.
	
	The rest of this paper is arranged as follows. In section II the
	mathematical preliminaries are presented along with stability
	analysis of the proposed filter and ultimate error bound
	determination. Section III concentrates on the model utilized to
	estimate both longitudinal and lateral positions of frontal dynamic
	objects. In section IV, simulations are performed on the comparison
	of the proposed filter with other methods on solving the SFM
	problem. Experimental results are reported in section V to verify
	the applicability of the proposed approach. Finally, concluding
	remarks are presented in section VI.
	
	\section{The Designed Filter Theory}
	\subsection{Preliminaries}
	Assume the general representation of a nonlinear system as follows.
	\begin{eqnarray} \label{1}
	&x(t+1) = f(x(t),u(t)) + \Delta f(x(t),u(t)) + Gw_0(t)\\
	&y(t) = h(x(t),u(t)) + \Delta h(x(t),u(t)) + D_1v_0(t)
	\end{eqnarray}
	where $x(t) \in \Re^n$ represents the system state vector, $u(t) \in \Re^m$
	is the input, $y(t) \in \Re^p$ indicates the observed output, $w_0
	(t)$ and $v_0 (t)$, are the process and measurement noises with unknown statistical
	properties, respectively; furthermore, $\Delta f(x(t),u(t))$ and $\Delta h(x(t),u(t))$
	are the model uncertainties. Note that, $t$ denotes the time-step.
	
	SDRE method is based on a state dependent
	coefficient (SDC) formed state space. According to \cite{ref24}, most non-linear
	state equations can be transformed to this form. Thus, rewriting the system model in SDC form yields:
	\begin{eqnarray}
	\label{3}
	&x(t+1) = A(x(t))x(t) + \Delta f(x(t)) + B(x(t))u(t) \nonumber \\
	&+ \Delta B(x(t))u(t) + Gw_0 (t)
	\end{eqnarray}
	\begin{eqnarray}
	\label{4}
	&y(t)=C(x(t))x(t) + \Delta h(x(t)) + D(x(t))u(t) \nonumber \\
	&+ \Delta D(x(t))u(t) + D_1v_0(t)
	\end{eqnarray}
	considering the following assumption holds : $\Delta f(x,u)= \Delta
	f(x) + \Delta B(x)u$ and $\Delta h(x,u)= \Delta h(x) + \Delta
	D(x)u$.
	
	\emph{ Remark 1:} SDC form is not unique for
	multivariable systems as if $A_1(x)$ and $A_2(x)$ are two distinct factor
	coefficients of $f(x)$, then, $A_3 (x) = Z(x) A_1 (x) + (I - Z(x) )
	A_2 (x)$ is a parameterizations of $f(x)$ for any matrix
	function $Z(x) \in \Re^{n \times n}$ as well. This flexibility is valid for $A(x)$, $B(x)$,
	$C(x)$ and $D(x)$.
	
	This Remark shows a significant advantage of SDRE method, which
	results in design process efficiency such as singular and
	unobservable points avoidance, and Lipschitz condition realization.
	The main objective is to design a filter which estimates the state
	vector $x(t)$ robustly, through measurable output vector, $y(t)$.
	The proposed filter equation is presented as follows.
	\begin{eqnarray}
	\label{5}
	&\hat{x}(t+1) = A(\hat{x})\hat{x}(t) + B(\hat{x})u(t)
	\nonumber \\
	& + L(\hat{x},t)[ y(t) - (C(\hat{x})\hat{x}(t) + D(\hat{x})u(t))]
	\end{eqnarray}
	where $\hat{x}(t)$ demonstrates estimated state variable vector and the
	filter gain $L(\hat{x},t) \in \Re^{n \times p}$, is defined as
	\begin{equation}
	\label{6}
	L(\hat{x},t) = A(\hat{x})P(t)C^{T}(\hat{x})[C(\hat{x})P(t)C^{T}(\hat{x})+R(t)]^{-1}
	\end{equation}
	in which, the symmetric matrix $P(t) \in \Re^{n \times n}$
	satisfies the state dependent differential Riccati
	equation~(SDDRE) (\ref{7}) with positive definite matrices $Q(t) \in
	\Re^{n \times n}$ and $R \in \Re^{p \times p}$  \cite{ref25}.
	\begin{eqnarray}
	\label{7}
	&P(t+1) = A(\hat{x})P(t)A^{T}(\hat{x}) + Q - A(\hat{x})P(t)C^{T}(\hat{x})[ \nonumber \\
	& C(\hat{x})P(t)C^{T}(\hat{x})+R(t)]^{-1}C(\hat{x})P(t)A^{T}(\hat{x})
	\end{eqnarray}
	The estimation error is defined as following:
	\begin{equation}
	\label{8}
	e(t) = x(t)- \hat{x}(t)
	\end{equation}
	Subtract(\ref{5}) from (\ref{3}) and perform some simplifications to
	reach to:
	\begin{eqnarray}
	\label{9}
	&e(t+1) = [A(\hat{x}) - L(\hat{x},t)C(\hat{x})]e(t) + \alpha (x,\hat{x},u)   \nonumber \\
	&- L(\hat{x},t)\beta (x,\hat{x},u) - L(\hat{x},t)v(t) + \Gamma w(t)
	\end{eqnarray}
	where,
	\begin{equation}
	\label{10}
	\alpha (x,\hat{x},u) = [A(x) - A(\hat{x})]x(t) + [B(x) - B(\hat{x})]u(t)
	\end{equation}
	\begin{equation}
	\label{11}
	\beta (x,\hat{x},u) = [C(x) - C(\hat{x})]x(t) + [D(x) - D(\hat{x})]u(t)
	\end{equation}
	and
	\begin{eqnarray}
	\label{111} 
	&w = (\Delta f(x) \quad \Delta B(x)u(t) \quad w_0(t))^{T}, \nonumber \\
	&\Gamma = (I \quad I \quad G ), \nonumber \\
	&v = (I \quad I \quad I)(\Delta h(x) \quad \Delta D(x)u(t) \quad D_1v_0)^{T}. \nonumber
	\end{eqnarray}
	
	As the final step to develop the  proposed filter, the switching
	concept is incorporated. It is proposed to switch $P(t)$ to its
	initial value with a specific frequency that is obtained through
	estimation error stability and ultimate error bound analysis. As a
	result, the filter estimation is divided into a family of
	subsystems, that can be studied in switched systems domain.
	
	\subsection{Stability Analysis}
	In this subsection, the sufficient condition for the estimation
	error dynamic stability is obtained. Furthermore, the ultimate
	estimation error bound caused by the model uncertainties is
	analytically acquired to further investigate the switching effects.
	Since time dependent switching is applied, the following dwell-time
	theorem is proven here for the sake of stability analysis.
	\newtheorem{Theorem}{Theorem}
	\begin{Theorem}
		Consider the non-linear discrete-time system presented in equations
		(\ref{3}) and (\ref{4}), along with equations (\ref{5}), (\ref{6})
		and (\ref{7}), related to the proposed  filter, and suppose the
		following assumptions hold:
		
		1) There exist $\overline{a},\overline{c}>0$, which $A(x)$ and
		$C(x)$ are upper-bounded for any $k \ge 0$ as follows:
		\begin{equation}
		\label{12} \Vert A(x) \Vert \le \overline{a},\;\; \Vert C(x) \Vert
		\le \overline{c}
		\end{equation}
		
		2) The states and inputs are bounded such that
		\begin{equation}
		\label{13} \Vert x(t) \Vert \le \sigma , \;\; \Vert u(t) \Vert \le
		\rho
		\end{equation}
		where $\sigma , \rho
		> 0$ for $t \ge 0$.
		
		3) $P(t)$, the solution of Riccati differential equation is bounded as follows
		\begin{equation}
		\label{14}
		\underline{p}I \le P(t) \le \overline{p}I
		\end{equation}
		where $\overline{p}$, $\underline{p}$ are positive real numbers.
		
		4) The SDC parameterizations is chosen such that matrices $A(x) ,
		B(x) , C(x) , D(x)$ are locally Lipschitz, i.e., there exist
		constants $k_A , k_B , k_C , k_D>0$ such that
		\begin{equation}
		\label{15}
		\Vert A(x) - A(\hat{x}) \Vert \le k_A \Vert x - \hat{x} \Vert
		\end{equation}
		\begin{equation}
		\label{16}
		\Vert B(x) - B(\hat{x}) \Vert \le k_B \Vert x - \hat{x} \Vert
		\end{equation}
		\begin{equation}
		\label{17}
		\Vert C(x) - C(\hat{x}) \Vert \le k_C \Vert x - \hat{x} \Vert
		\end{equation}
		\begin{equation}
		\label{18}
		\Vert D(x) - D(\hat{x}) \Vert \le k_D \Vert x - \hat{x} \Vert
		\end{equation}
		for any $x , \hat{x} \in \Re^{n}$ with $\Vert x - \hat{x} \Vert \le
		\varepsilon_A$ and $\Vert x - \hat{x} \Vert \le \varepsilon_B$ and
		$\Vert x - \hat{x} \Vert \le \varepsilon_C$ and $\Vert x - \hat{x} \Vert
		\le \varepsilon_D$, respectively.
		
		Then, to have stable estimation error dynamic (\ref{9}), the following condition is adequate to be hold for positive numbers $\kappa^\prime$ and $\lambda$.
		\begin{equation}
		\label{19}
		2 \kappa^\prime \Vert \Pi(t+1) [ - L(\hat{x},t)v(t) + \Gamma w(t)] \Vert + (\lambda_{0} + \frac{\lambda}{2}) \lambda_{max}^{\Pi (t)} \le \lambda \lambda_{min}^{\Pi (t)}
		\end{equation}
		where $\lambda_0$
		is directly related to the switching frequency and $\Pi (t) =
		P^{-1}(t)$.
	\end{Theorem}
	
	\emph{Remark 2:} The maximum and minimum eigenvalues of $\Pi$ are
	represented by $\lambda_{max}^{\Pi}$ and $\lambda_{min}^{\Pi}$,
	respectively.
	
	Before presenting the proof of theorem 1, some points have to be mentioned.
	
	\emph{Remark 3:} Inequality (\ref{14}) is very important to be
	satisfied in estimation error stability. This is directly related to
	the system observability based on the following assumptions
	\cite{ref26}:
	
	1- The design matrix $Q$ and initial condition $P(0)$ are positive
	definite and $A(x)$ has a bounded norm.
	
	2- The SDC form is chosen such that the pair of $\{A(x),C(x)\}$ is
	uniformly detectable through Definition 1.
	
	\vspace{2mm}
	
	\newtheorem{Definition}{Definition}
	\begin{Definition}
		$\{A(x),C(x)\}$ satisfies the  uniform observability condition, if
		there exist $\underline{m}, \overline{m}>0$ and $a>0$ as an integer
		such that
		\begin{equation}
		\label{20}
		\underline{m} I \le \sum_{i=t}^{t+a} \chi^{-T}(i,t) C_i^T C_i \chi(i,t) \le \overline{m} I
		\end{equation}
		where $\chi(t,t)=I$ and for $i>t$ the following is valid
		\begin{equation}
		\label{21}
		\chi(i,t) = A_{i-1}...A_t
		\end{equation}
		
	\end{Definition}
	
	\emph{Remark 4:} Assumption 4 is commonly considered in researches
	\cite{ref27} and compared to the previous SDRE filter studies, there
	is no new limiting condition for the mentioned SDC form.
	
	\emph{Remark 5:} Switching the covariance matrix yields to a family
	of subsystems which their stability can be analyzed by considering
	distinct Lyapunov functions for each subsystem based on direct
	Lyapunov method.
	
	\begin{proof}
		Theorem 1 is proved within two steps through exploiting the
		following Lemma.
		\newtheorem{Lemma}{Lemma}
		\begin{Lemma}
			Consider (\ref{9}) as a set of subsystems. Suppose there exist a
			Lyapunov candidate function $V_p$ for subsystem $p$ along with two
			$K_{\infty}$-class functions $\alpha_1$ and $\alpha_2$, and two
			constants $0 < \lambda_0 < 1$ and $\mu > 0$ such that
			\begin{equation}
			\label{22}
			\alpha_1 (\Vert x \Vert) \le V_p(x) \le \alpha_2 (\Vert x \Vert) \quad \forall x, \quad \forall p \in \varrho
			\end{equation}
			\begin{equation}
			\label{23}
			\Delta V_p(x) \le -\lambda_{0} V_p(x) \quad \forall x, \quad \forall p \in \varrho
			\end{equation}
			\begin{equation}
			\label{24}
			V_{p} (x) \le \mu V_{q} (x) \quad \forall x, \quad \forall p,q \in \varrho
			\end{equation}
			where, $\varrho$ denotes the set of all subsystems. Then, the
			non-linear discrete-time switched system (\ref{9}) is stable for
			every switching signal with average dwell-time $\tau_a$ if
			\begin{equation}
			\label{25}
			\tau_a > -\frac{\ln{\mu}}{\ln(1 - \lambda_0)}
			\end{equation}
		\end{Lemma}
		Lemma1 is comprehensively studied along with its proof in
		\cite{ref28}. Since inequality (\ref{22}) holds due to assumption 3,
		the two following steps focus on ensuring conditions (\ref{23}) and
		(\ref{24}).
		
		\emph{Step 1:} To present the sufficient condition for (\ref{23}),
		the following Lyapunov function candidate  is considered:
		\begin{equation}
		\label{26}
		V_p(e(t))=e^T(t)\Pi(t)e(t)
		\end{equation}
		Furthermore, considering (\ref{14}) results to:
		\begin{equation}
		\label{27}
		\frac{1}{\overline{p}} {\Vert e(t) \Vert}^{2} \le V(e(t)) \le \frac{1}{\underline{p}} {\Vert e(t) \Vert}^{2}
		\end{equation}
		where $\overline{p}$ and $\underline{p}$ are the minimum and maximum
		eigenvalues of matrix $P$, respectively. Inequality (\ref{27})
		indicates that the selected Lyapunov function candidate  is positive
		definite and decresent. Next, inequality (\ref{23}) is investigated
		by considering (\ref{26}) time-difference.
		\begin{eqnarray}
		\label{28}
		& \Delta V_p(e) = e^T(t+1)\Pi(t+1)e(t+1) - e^T(t)\Pi(t)e(t) \nonumber \\
		& = [\alpha - L \beta - L v + \Gamma w]^T \Pi(t+1) [\alpha - L \beta - L v + \Gamma w]  \nonumber \\
		&  - e^T \Pi(t) e + e^T[A - LC]^T\Pi(t+1)[A - LC]e \nonumber \\
		& + 2[\alpha - L \beta - L v + \Gamma w]^T \Pi(t+1) [A - L C]e
		\end{eqnarray}
		To simplify the above equation, the following two Lemmas are needed.
		
		\begin{Lemma}
			Under the four mentioned assumptions in theorem 1, there exist
			$0<\lambda<1$ such that the following inequality is satisfied if
			(\ref{6}) is used:
			\begin{equation}
			\label{29}
			(A - LC)^T \Pi(t + 1) (A - LC) \le (1 - \lambda) \Pi(t)
			\end{equation}
		\end{Lemma}
		
		\begin{proof}
			Employing (\ref{6}) and (\ref{7}) yields to
			\begin{equation}
			\label{30}
			P(t+1) = (A - LC) P (A - LC)^T + Q + LCP (A - LC)^T
			\end{equation}
			Considering (\ref{6}) and $LCP (A - LC)^T$, it can be shown that
			\begin{equation}
			\label{31}
			A^{-1}(A - LC)P = P - PC^T (CPC^T + R)^{-1}CP
			\end{equation}
			Moreover, utilize the inverse matrix Lemma to reach
			\begin{equation}
			\label{32}
			A^{-1}(A - LC)P = (P^{-1} + C^T R^{-1} C)^{-1}
			\end{equation}
			Equation (\ref{32}) is positive definite due to the fact that
			$P^{-1}(t)>0$ and
			\begin{equation}
			\label{33}
			A^{-1}LC = PC^T (CPC^T + R)^{-1}C \ge 0
			\end{equation}
			Then considering (\ref{32}) along with (\ref{33}) and $P^T=P$, it
			can be shown that
			\begin{equation}
			\label{34}
			LCP(A - LC)^T = A(A^{-1}LC)(A^{-1}(A - LC)P)^T
			\end{equation}
			Replace (\ref{34}) in (\ref{30}):
			\begin{equation}
			\label{35}
			P(t+1) \ge (A - LC) [P + (A - LC)^{-1}Q(A - LC)^{-T}](A - LC)^T
			\end{equation}
			On the other hand, taking (\ref{6}) along  with $\underline{q}I \le
			Q$, $\underline{r}I \le R$ yields to
			\begin{equation}
			\label{36}
			\Vert L \Vert \le \overline{apc} \frac{1}{\underline{r}}
			\end{equation}
			Hence, the following inequality is obtained, by using (\ref{36})
			alongside (\ref{35}) and (\ref{14}).
			\begin{equation}
			\label{37}
			(A - LC)^T \Pi(t+1) (A - LC) \le (1+\frac{\underline{q}}{\overline{p}(\overline{a}+\overline{apc}^2/\underline{r})^2})^{-1}\Pi(t)
			\end{equation}
			Finally, the inequality (\ref{29}) is held with the following
			relation:
			\begin{equation}
			\label{38}
			1 - \lambda = (1+\frac{\underline{q}}{\overline{p}(\overline{a}+\overline{apc}^2/\underline{r})^2})^{-1}
			\end{equation}
		\end{proof}
		
		\begin{Lemma}
			Consider that the above mentioned assumptions are satisfied, then
			there exist $\varepsilon^\prime , k_\varphi > 0$ such that for
			$\Vert e(t) \Vert \le \varepsilon^\prime$, the following inequality
			holds.
			\begin{equation}
			\label{39}
			[\alpha - L\beta]^T \Pi(t+1) [2(A - LC)e + \alpha - L\beta] \le k_\varphi \Vert e \Vert^3
			\end{equation}
			in which, $\alpha$ and $\beta$ are in accordance with (\ref{10}) and
			(\ref{11}), respectively.
		\end{Lemma}
		
		\begin{proof}
			Utilizing (\ref{36}) results in
			\begin{equation}
			\label{40}
			\Vert \alpha - L\beta \Vert \le \Vert \alpha \Vert + \overline{apc}\frac{1}{\underline{r}} \Vert \beta \Vert
			\end{equation}
			Now consider the assumption four to reach to
			\begin{equation}
			\label{41}
			\Vert \alpha \Vert \le (k_A \sigma + k_B \rho) \Vert x - \hat{x} \Vert
			\end{equation}
			\begin{equation}
			\label{42}
			\Vert \beta \Vert \le (k_C \sigma + k_D \rho) \Vert x - \hat{x} \Vert
			\end{equation}
			Then for $\varepsilon^\prime = min(\varepsilon_A,\varepsilon_B,\varepsilon_C,\varepsilon_D)$, (\ref{40}) can be rewritten as
			\begin{equation}
			\label{43}
			\Vert \alpha - L\beta \Vert \le (k_A \sigma + k_B \rho) \Vert e \Vert^2 + \overline{apc}\frac{1}{\underline{r}}(k_C \sigma + k_D \rho)\Vert e \Vert^2
			\end{equation}
			to further simplify, for $\Vert e \Vert \le \varepsilon^\prime$, the
			following inequality holds
			\begin{equation}
			\label{44}
			\Vert \alpha - L\beta \Vert \le k^\prime \Vert e \Vert^2
			\end{equation}
			where
			\begin{equation}
			\label{45}
			k^\prime = (k_A \sigma + k_B \rho) + \overline{apc}\frac{1}{\underline{r}}(k_C \sigma + k_D \rho)
			\end{equation}
			Hence, inequality (\ref{39}) is satisfied through exerting
			(\ref{46}).
			\begin{equation}
			\label{46}
			k_\varphi = \frac{k^\prime}{\underline{p}}[2(\overline{a} + \overline{apc}^2 / \underline{r}) + k^\prime \varepsilon^\prime]
			\end{equation}
		\end{proof}
		
		Applying the above two presented Lemmas to (\ref{28}) yields to:
		\begin{eqnarray}
		\label{47}
		& \Delta V_p(e) \le k_\varphi \Vert e \Vert^3 - (\lambda \lambda_{min}^{\Pi(t)}) \Vert e \Vert^2 + \lambda_{min}^{\Pi_{t+1}} \Vert - L v + \Gamma w \Vert \nonumber \\
		& + 2 \Vert [- L v + \Gamma w]^T \Pi_{t+1} [A - LC]\Vert \Vert e \Vert \nonumber \\
		& + 2 \Vert [\alpha - L\beta]^T \Pi_{t+1} [- L v + \Gamma w] \Vert
		\end{eqnarray}
		On the other hand, taking (\ref{27}) along with  $\Vert e \Vert \le
		\varepsilon$ yields to
		\begin{equation}
		\label{48}
		k_\varphi \Vert e \Vert^3 \le \frac{\lambda}{2}V_p(e)
		\end{equation}
		Finally, replacing (\ref{48}) and (\ref{44}) into (\ref{47}) results
		\begin{eqnarray}
		\label{49}
		& \Delta V_p(e) \le \lambda_{min}^{\Pi_{t+1}}\Vert  - L v + \Gamma w \Vert \nonumber \\
		& + (2k^\prime \Vert \Pi(t+1)[- L v + \Gamma w] \Vert - \lambda(\lambda_{min}^{\Pi(t)} - \frac{\lambda_{max}^{\Pi(t)}}{2}))\Vert e \Vert^2  \nonumber \\
		& + 2 \Vert [- L v + \Gamma w]^T \Pi(t+1) [A - LC]\Vert  \Vert e \Vert
		\end{eqnarray}
		Thus, sufficient condition to have (\ref{23}) is (\ref{19}). \\
		\emph{Remark 6:} The following well-known inequality is utilized along the proofs.
		\begin{equation}
		\label{50}
		\lambda_{min}^{A} \Vert e \Vert^2 \le e^T A e \le \lambda_{max}^{A} \Vert e \Vert^2
		\end{equation}
		where $A$ is supposed to be an arbitrary matrix.
		
		\emph{Step 2:} Consider $V_p$ as the Lyapunov function of the
		subsystem $p$. Suppose that $V_1 = e_1^T \Pi_{t_1} e_1$ and $V_2 =
		e_2^T \Pi_{t_2} e_2$ are the considered $V_p$ values at times $t_1$
		and $t_2$, respectively. Furthermore, assume that the inequality
		$\Pi_{t_2} > \Pi_{t_1}$ is valid. Consequently,  if (\ref{19})
		holds, then the inequality $e_1 > e_2$ is surely authentic for
		errors higher than the steady state error. Now if the switching
		happens at time $t_3$, taking $V_3 = e_3^T \Pi_{t_3} e_3$ along with
		$e_1 > e_2
		> e_3$ and $\Pi_{t_1} = \Pi{t_3}$ (as the result of switching), the
		inequality (\ref{24}) is immediately satisfied.
	\end{proof}
	Here, the proof of theorem 1 is accomplished.  However, the
	following remarks have to be considered.\\
	\emph{ Remark 7:} Taking ({\ref{14}), assumption $\Pi_{t_2} > \Pi_{t_1}$
		is logical.  \\
		\emph{ Remark 8:} Since the stability is analyzed in presence of model
		uncertainties, the estimation error is locally stable.\\
		\emph{ Remark 9:} To determine the proper switching frequency, a
		comprehensive study on the switching effect is needed. To this end,
		the following points shall be considered:\\
		1- The sufficient condition for (\ref{23}) to be valid is
		(\ref{19}); as it is seen, the first effect would be to hold
		$\lambda_{min}^{\Pi(t)}$ large and $\lambda_{max}^{\Pi(t)}$ small
		enough. Furthermore, there are two other parameters in (\ref{19})
		which are $k^{\prime}$ and $\lambda$. Since switching maintains
		$\overline{p}$ small, $k^{\prime}$ would be minimized according to
		(\ref{45}). In addition, enlarging $\lambda$ can reduce the
		uncertainty term effect in (\ref{19}), and by switching this
		conditions are fulfilled as a result of minimizing $\overline{p}$.\\
		2- From the previous point, one can infer to set the switching
		frequency  as large as it can. However, there is a significant
		compromise; considering (\ref{25}),  increasing switching frequency
		yields enhancement of $\lambda_0$ in (\ref{19}) which may cause some
		problems. Thus, the switching frequency have to be specified in a
		certain interval.
		
		\subsection{Ultimate Error Bound Analysis}
		This section concentrates on obtaining analytic form of ultimate
		estimation error bound caused by the model uncertainties. To this
		end, consider (\ref{49}) which is a quadratic function of $\Vert
		e(t) \Vert$. If (\ref{19}) holds, this function attains its maximum
		value of
		\begin{eqnarray}
		\label{51}
		&\epsilon = -[\frac{(\Vert [- L v + \Gamma w]^T \Pi(t+1) [A - LC])^2}{(2k^\prime \Vert \Pi(t+1)[- L v + \Gamma w] \Vert - \lambda(\lambda_{min}^{\Pi(t)} - \frac{\lambda_{max}^{\Pi(t)}}{2}))}] \nonumber \\
		& + \lambda_{min}^{\Pi_{t+1}}\Vert  - L v + \Gamma w \Vert
		\end{eqnarray}
		at
		\begin{equation}
		\label{52}
		\Vert e(t) \Vert =  -\frac{\Vert [- L v + \Gamma w]^T \Pi(t+1) [A - LC] \Vert}{(2k^\prime \Vert \Pi(t+1)[- L v + \Gamma w] \Vert - \lambda(\lambda_{min}^{\Pi(t)} - \frac{\lambda_{max}^{\Pi(t)}}{2}))}
		\end{equation}
		Hence, (\ref{49}) can be rewritten as
		\begin{equation}
		\label{53}
		\Delta V_p (e(t)) \le - \alpha_3 (\Vert e(t) \Vert) + \epsilon
		\end{equation}
		where $\alpha_3$ is a positive definite function. The following
		theorem determines the ultimate error bound in this case.
		\begin{Theorem}
			Consider the Lyapunov function defined in (\ref{26}) along with
			inequalities (\ref{27}) and (\ref{53}); suppose that (\ref{19}) is
			also satisfied, then the ultimate estimation error bound is given
			by:
			\begin{equation}
			\label{54} \left (\frac{\overline{p}}{\underline{p}}\right)^{1/2}
			\cdot \Vert (\alpha_3^{-1}(\epsilon)) \Vert
			\end{equation}
		\end{Theorem}
		\begin{proof}
			Since (\ref{19}) holds, setting $\Delta V_p$ equal to zero in
			(\ref{53}) yields an equation which its solution is the ultimate
			error bound. Consequently, considering (\ref{49}) beside (\ref{53})
			results to the following solution:
			\begin{equation}
			\label{55}
			\frac{-b + \vert (b^2 - 4ac)^{1/2} \vert}{2a}
			\end{equation}
			where
			\begin{equation}
			\label{56}
			a = -2k^\prime \Vert \Pi(t+1)[- L v + \Gamma w] \Vert + \lambda(\lambda_{min}^{\Pi(t)} - \frac{\lambda_{max}^{\Pi(t)}}{2})
			\end{equation}
			\begin{equation}
			\label{57}
			b = -2 \Vert [- L v + \Gamma w]^T \Pi(t+1) [A - LC] \Vert
			\end{equation}
			\begin{equation}
			\label{58}
			c = -\lambda_{min}^{\Pi_{t+1}}\Vert  - L v + \Gamma w \Vert
			\end{equation}
			On the other hand, the following inequality is surely held
			\cite{ref29}
			\begin{equation}
			\label{59}
			\Vert e(t) \Vert \le \alpha_1^{-1}(\alpha_2(\alpha_3^{-1})) \quad \forall t \ge kT + t_0
			\end{equation}
			where $\alpha_1$ and $\alpha_2$ are obtained from (\ref{27})
			according to (\ref{22}). Thus, the result is equal to (\ref{54}) and
			the proof of theorem 2 is completed.
		\end{proof}
		\emph{ Remark 10:} As expressed  in remark 9, the impact of
		switching is to reduce $\overline{p}$, and increase $\underline{p}$.
		Consequently, taking the ultimate bound of estimation error as
		(\ref{54}), switching alleviates the steady state error through
		minimizing $\left (\frac{\overline{p}}{\underline{p}}\right )$.
		
		\section{The Employed Model}
		\subsection{Kinematic Motion Model}
		The state space model utilized to implement the proposed filter is
		presented in this section. First, a general form is introduced along
		with a simple output model. Then, to ensure observability of the
		model in an autonomous vehicle motion plane, a modified output model
		is proposed based on the intercept theorem \cite{ref30}. Fig.
		\ref{fig1} delineates coordinate systems for a moving object
		observed by a moving camera. where $\mathcal{F}_G$ represents a
		fixed inertial reference frame and $\mathcal{F}_C$ denotes the
		camera fixed reference frame. The vectors $r_q,r_c \in \Re^3$ are
		from the origin of $\mathcal{F}_G$ to the point on the object and
		the camera principle point, respectively. Consider $r_{q/c}
		\triangleq [X \quad Y \quad Z]^T \in \Re^3$ as the relative position
		vector represented in $\mathcal{F}_C$, the following expresses its
		kinematic
		\begin{equation}
		\label{60}
		\dot{r}_{q/c} = v_q - v_c - \omega \times r_{q/c}
		\end{equation}
		where $v_q \in \Re^3$ denotes the linear velocity of the point on
		the object, $v_c \in \Re^3$ is the linear camera velocity and
		$\omega \in \Re^3$ is the angular velocity of the camera, all
		represented in $\mathcal{F}_C$.
		\begin{figure}[t]
			\centering
			\includegraphics[width=2.5in]{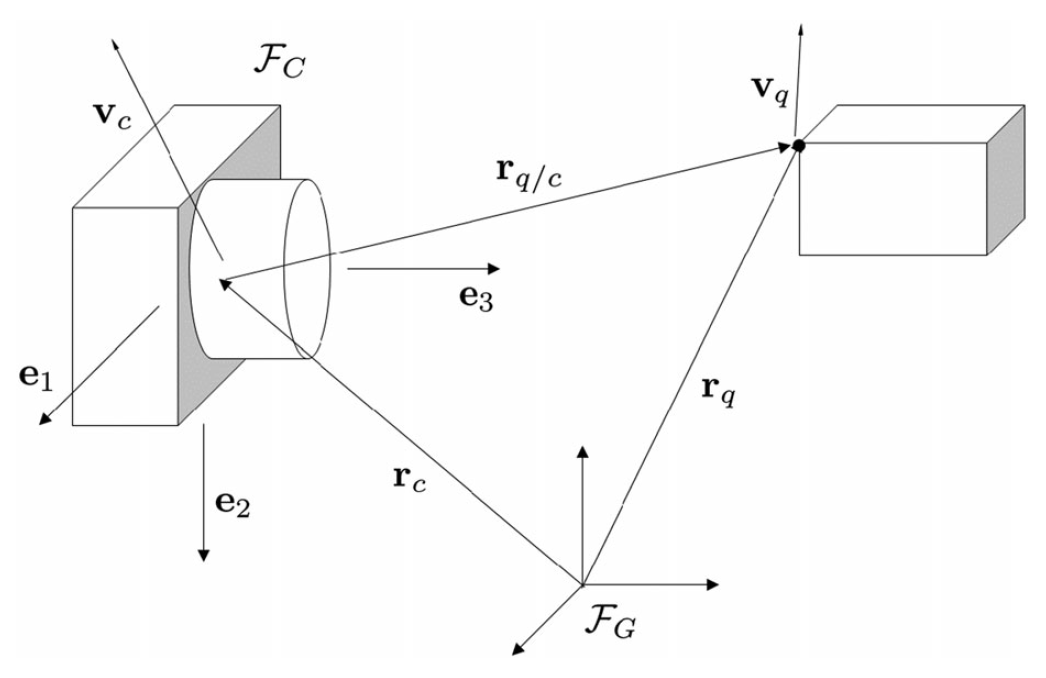}
			\caption{Schematic of the Model \cite{ref10}.}
			\label{fig1}
		\end{figure}
		Taking $x = [x_1, x_2, x_3]^T = [\frac{X}{Z}, \frac{Y}{Z}, \frac{1}{Z}]^T \in \Re^3$ as the state vector, the continuous-time state space model is obtained as follows
		\begin{eqnarray}
		\label{61}
		&\dot{x}_1 = \Omega_1 + \zeta_1 + v_{q1} x_3 - x_1 v_{q3} x_3 \nonumber \\
		&\dot{x}_2 = \Omega_2 + \zeta_2 + v_{q2} x_3 - x_2 v_{q3} x_3 \nonumber \\
		&\dot{x}_3 = v_{c3} x_3^2 - (\omega_2 x_1 - \omega_1 x_2) x_3 - v_{q3} x_3^2
		\end{eqnarray}
		where it is assumed that $v_q\triangleq[v_{q1},v_{q2}, v_{q3}]$,
		$v_c\triangleq[v_{c1},v_{c2},v_{c3}]$, $\omega
		\triangleq[\omega_{1},\omega_{2},\omega_{3}]$; besides, $\Omega_1,
		\Omega_2, \zeta_1, \zeta_2: [0,\infty) \times \Re^3 \rightarrow \Re$
		are defined as follows
		\begin{eqnarray}
		\label{62}
		&\Omega_1 (t,x) = \omega_3x_2 - \omega_2 - \omega_2x_1^2 + \omega_1 x_1 x_2 \nonumber \\
		&\Omega_2 (t,x) = -\omega_3x_1 + \omega_1 - \omega_2x_1x_2 + \omega_1 x_2^2 \nonumber \\
		&\zeta_1(t,x) = (v_{c3}x_1 - v_{c1}) x_3 \nonumber \\
		& \zeta_2(t,x) = (v_{c3}x_2 - v_{c2}) x_3
		\end{eqnarray}
		To get comprehensive details about the presented model,  refer to
		\cite{ref10}. Since the images taken from camera are utilized as
		observations, the output model is simply $y = [x_1,x_2]^T$ as well.
		
		The  following assumptions have to be made, in order to employ this model for the proposed filter:\\
		1- Camera velocities $v_c$ and $\omega$ are measurable.\\
		2- The first two state variables are bounded and observed continuously by the camera.\\
		3- Object velocities are taken as unknown inputs to the system. \\
		\emph{ Remark 11:} Utilizing projective geometry,  the feature point
		coordinates of the image, $u,v$ are related to the normalized
		Euclidean one by the following equation:
		\begin{equation}
		\label{63}
		[u \quad v \quad 1]^T = A_c [Y/Z \quad X/Z \quad 1]^T
		\end{equation}
		where $A_c \in \Re^{3 \times 3}$ is the invertible camera  intrinsic
		parameter matrix \cite{ref10}. Utilizing (\ref{63}), the states
		$x_1$ and $x_2$ are measurable if assumption 2 holds. This model has
		a major limitation in 2D motion plane of an autonomous vehicle;
		having just one state variable, $x_2$ observed by the camera,
		estimating the depth ($x_3$) is not possible due to the resulted
		unobservable model. To solve this problem, a new output model is
		defined through employing the intercept theorem.
		\subsection{The Developed Model}
		In an autonomous vehicle motion plane, there are two  degrees of
		freedom. Since the presented kinematic model is generally in 3D
		dimension, it has to be reduced to a 2D one which does not consider
		$x_1$. In this regard, $x_1$ can be set as a constant known value in
		(\ref{61}). To have an observable state space, an output model is
		developed through applying the following theorem.
		\begin{Theorem}
			(Thales' intercept theorem \cite{ref30}). Consider an arbitrary triangle $ABC$ as in Fig. \ref{fig2} and let $AC$ be extended to $C^\prime$ and $AB$ to $B^\prime$, so that $B^\prime C^\prime$ is parallel to $BC$. Then the lengths of the sides satisfy the relations
			\begin{equation}
			\label{64}
			\frac{a^\prime}{a} = \frac{b^\prime}{b} = \frac{c^\prime}{c}
			\end{equation}
			and hence
			\begin{equation}
			\label{65}
			\frac{a^\prime}{c^\prime} = \frac{a}{c},\quad \frac{c^\prime}{b^\prime} = \frac{c}{b},\quad \frac{b^\prime}{a^\prime} = \frac{b}{a}
			\end{equation}
		\end{Theorem}
		\begin{figure}[t]
			\centering
			\includegraphics[width=2.5in]{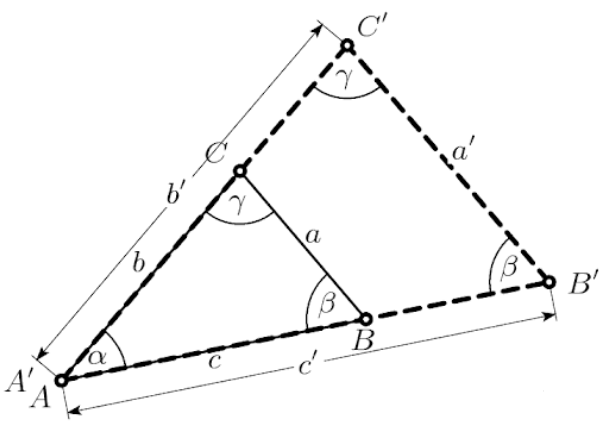}
			\caption{Thales' intercept theorem \cite{ref30}.}
			\label{fig2}
		\end{figure}
		Now consider Fig. \ref{fig3} to carry on the output model expansion. The main purpose is to estimate $Z$, which is the distance from the object contact point with the ground to the camera pinhole $P$. In this regard, $Y$ is taken as an observation on the image plane and the following relation is obtained through employing theorem 3:
		\begin{equation}
		\label{66}
		\frac{F}{F+Z} = \frac{\tilde{Y}}{\tilde{Y}+H}
		\end{equation}
		where $F$ is the focal length of the camera, $H$ is the height of
		the camera with respect to the ground, and $Y_1=Y-VP$. Note that,
		$Y_1$ is not the observation since it is calculated with respect to
		the camera center point. Thus, $VP$ is used to modify $Y$. Furthermore, since $Y_1$ is in pixel unit, $\tilde{Y} = Y_1 \frac{h_f}{h}$ is utilized as the vertical position of the object in the image plane in which $h$ and $h_f$ are the maximum amount of pixels along the height of the camera film and its real height, respectively. Thus, having $F$, $H$ and $Y_1$, the depth($Z$) is estimated. However, to utilize this equation as the output model, rewrite (\ref{66}) as follows
		\begin{equation}
		\label{67} F\left (\frac{H}{Z}\right ) = \tilde{Y}
		\end{equation}
		
		\begin{figure}[]
			\centering
			\includegraphics[width=2.5in]{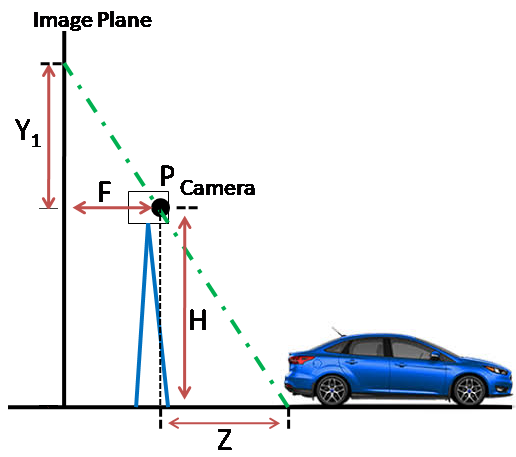}
			\caption{Concept of the developed output model}
			\label{fig3}
		\end{figure}
		\emph{Remark 12:} $h$ and $h_f$ are directly related to the sensor
		dimensions, and they are reported in the camera information.
		Besides, $VP$ is the vanishing point of the camera. 
	}
	
	Since the state vector is defined as $x = (\frac{X}{Z}, \frac{Y}{Z}, \frac{1}{Z})$, it can be seen that $Z = \frac{1}{x_3}$. To obtain $\tilde{Y}$ in terms of the state variables, consider matrix $A_c$ in (\ref{63}) as
	\begin{eqnarray}
	\label{67_1}
	&A_c =
	\left[
	\begin{matrix}
	f_x & 0 & x_0 \\ 0 & f_y & y_0 \\ 0 & 0 & 1
	\end{matrix}
	\right]
	\end{eqnarray}
	where $f_x$ and $f_y$ are equal to each other and representing the focal length of the camera. Moreover, $x_0$ and $y_0$ are determined from the distance between the image origin and its center. Then taking into account (\ref{67_1}) along (\ref{63}), $\tilde{Y}$ is obtained as follows:
	\begin{equation}
	\label{67_2}
	\tilde{Y} = (f_y(x_1) + y_0)\frac{h_f}{h}
	\end{equation}
	Consequently, the output model would be
	\begin{equation}
	\label{68}
	x_1 = ((FHx_3 )\frac{h}{h_f} - y_0)\frac{1}{f_y}
	\end{equation}
	There is an important point  to note; although $x_1$ is taken as a
	constant value in the state space, the output model can be defined
	as follows without loss of generality:
	\begin{equation}
	\label{69}
	y = [x_2, ((FHx_3)\frac{h}{h_f} - y_0)\frac{1}{f_y} ]^T
	\end{equation}
	Note that $x_1$ is not considered in (\ref{69}) due to the
	assumption made to take this state as a constant known parameter.
	This output model yields an observable state space which can be
	employed in the real implementation.
	
	\section{Simulation Results}
	This section is organized as follows: First, simulation results on
	state estimation based on the general state space form in (\ref{61})
	alongside $y = [x_1, x_2]^T$ as observation model are analyzed in
	absence of model uncertainties. Then, the uncertainty in the model
	is added and its effects are studied. Finally, a Monte Carlo
	simulation is performed to analyze the robustness of the proposed
	approach with respect to the other approaches. All simulations are
	reported based on a comparison study of the proposed filter with an
	unknown input observer and the common SDRE filter. A point to
	ponder is that the UIO is a common approach in solving the SFM
	problem in the literature \cite{ref10,ref11}.
	Furthermore, this method is designed for estimations in the presence
	of unknown inputs \cite{ref10}. Since in this paper the problem is
	discussed in the presence of unknown inputs in addition to model
	uncertainties, it is proposed to utilize the proposed switched SDRE
	method instead. Note that, simulation details are in accordance
	with~\cite{ref10}.
	
	\begin{figure}[]
		\centering
		\subfloat[]{\includegraphics[width=2in]
			{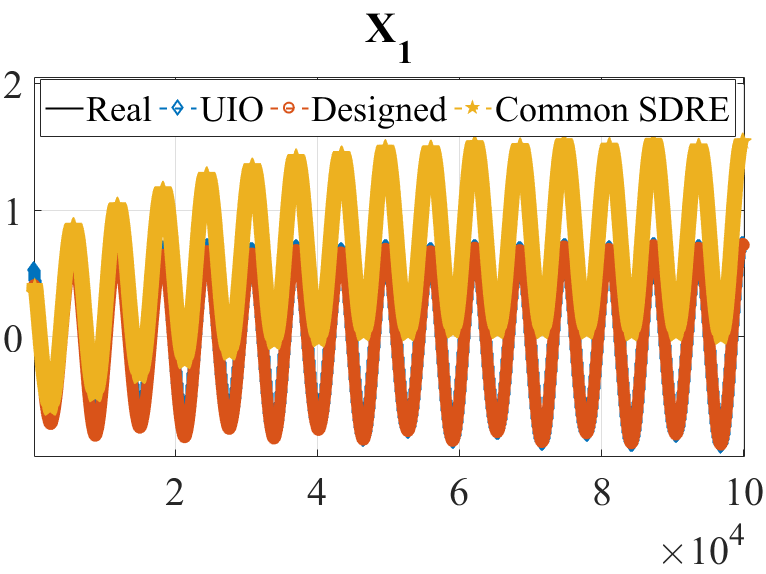}}
		\subfloat[Magnified]{\includegraphics[width=2in]
			{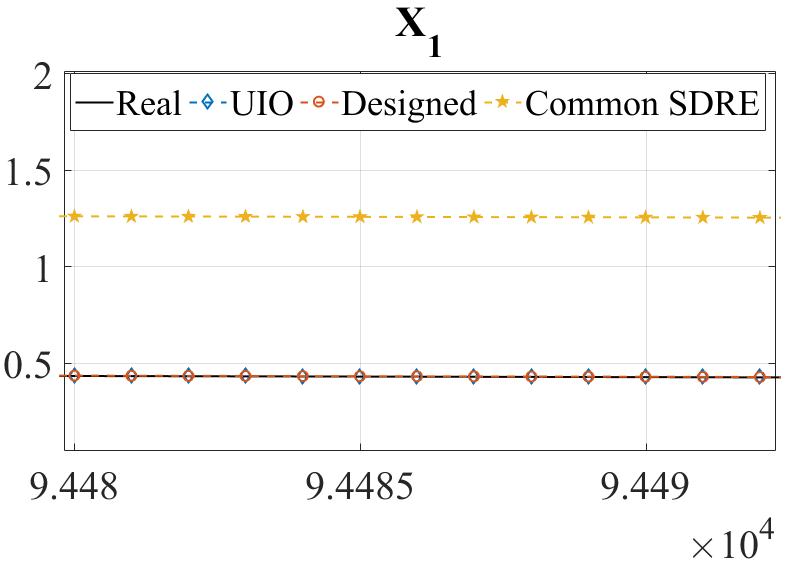}} \\
		\subfloat[]{\includegraphics[width=2in]{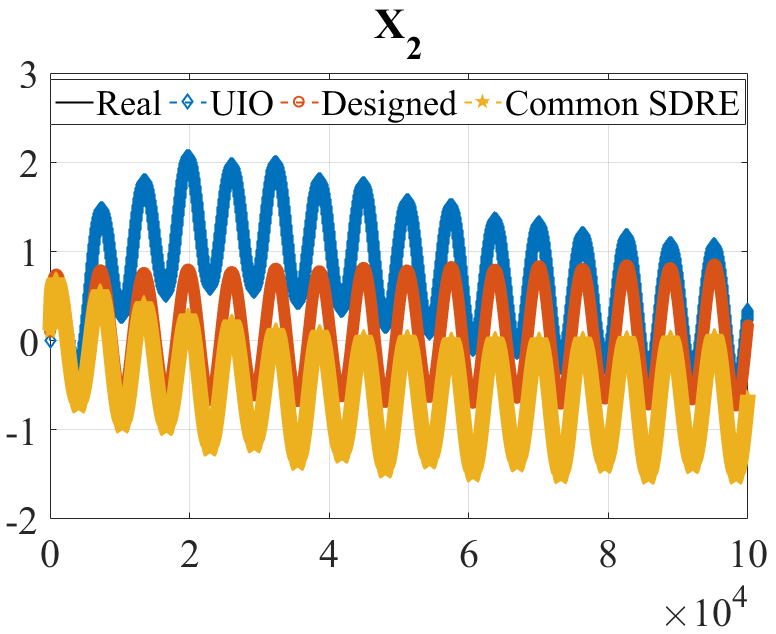}}
		\subfloat[Magnified]{\includegraphics[width=2in]
			{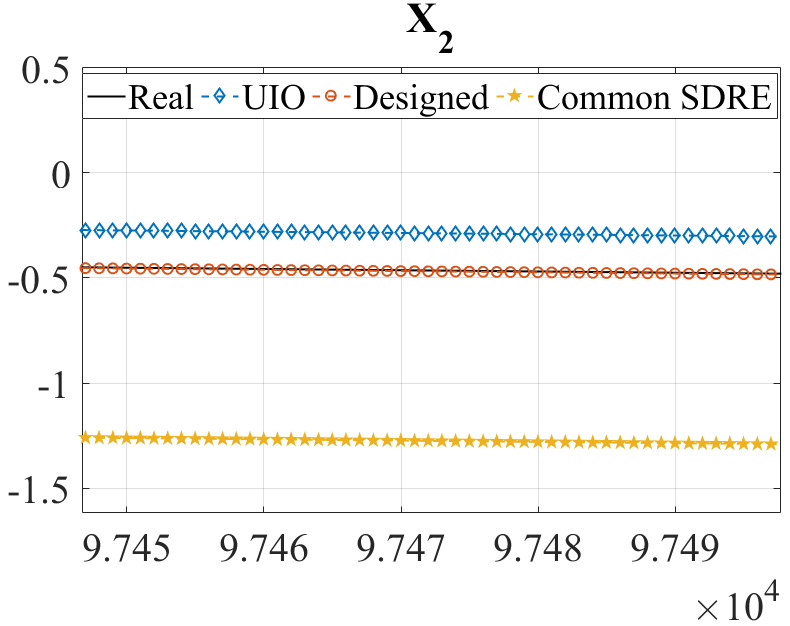}}\\
		\subfloat[]{\includegraphics[width=2.1in]{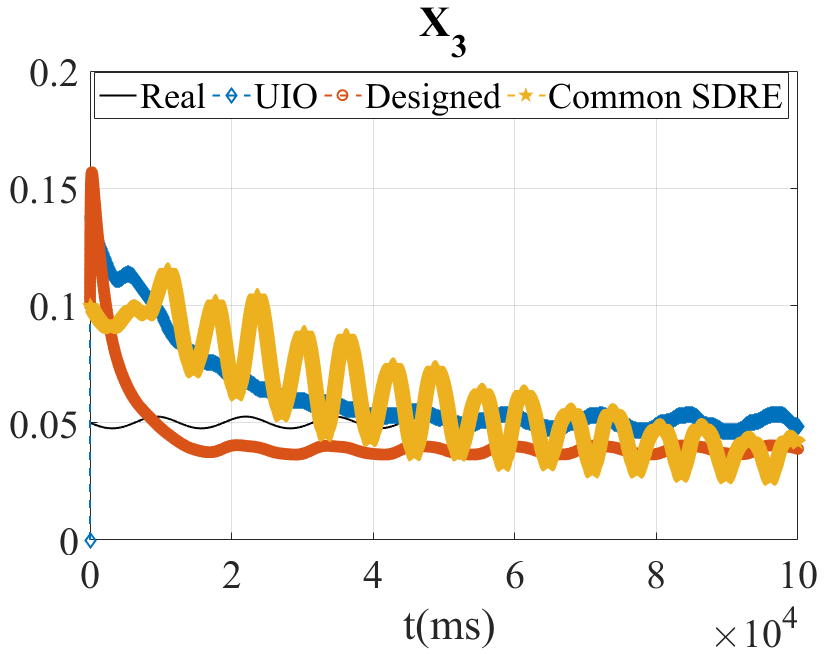}}
		\subfloat[Magnified]{\includegraphics[width=2.1in]
			{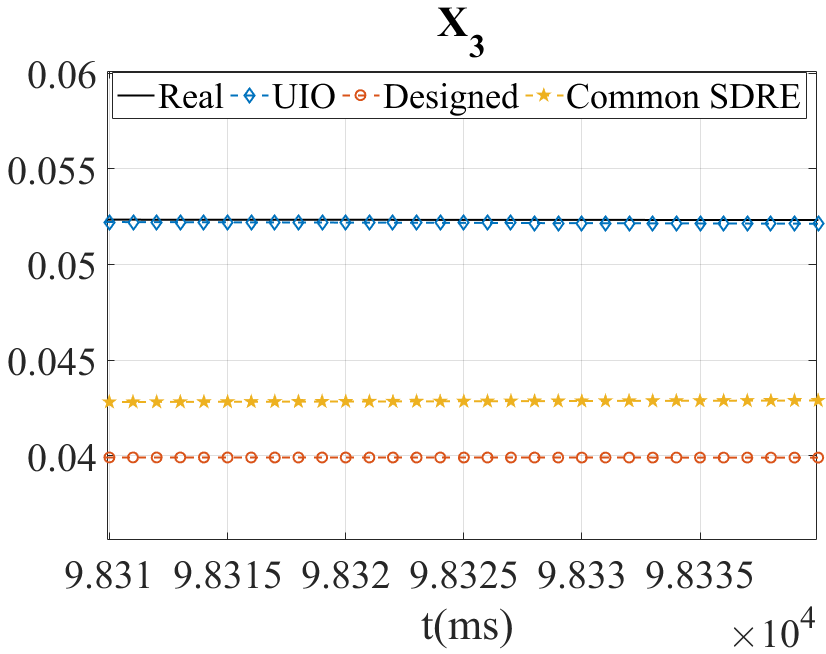}} \caption{Estimation results in absence of uncertainties of the model}
		\label{fig4}
	\end{figure}
	\emph{ Remark 13:} To discretize the continuous-time model in
	(\ref{61}), Euler method \cite{ref31} is employed to preserve
	stability in the conversion.
	
	Fig. \ref{fig4} displays the results of the three mentioned methods
	for estimation of an object Euclidean position where the
	uncertainties are not present. As it can be seen in the simulation
	results, especially looking into the magnified plots, the UIO
	performance is better compared to the other approaches due to the
	complete elimination of disturbance effects. This is reported to
	ensure the efficiency of the UIO method in an ideal condition.
	Furthermore, improvement in the common SDRE filter operation by
	utilizing the switching concept in the designed approach is also
	clear specially in $x_1$ and $x_2$ states. Note that although there
	is no uncertainty in the model, the external disturbance results in
	an error in estimations which can be minimized through switching of
	covariance matrix in the proposed filter. Next, to take into account
	the model uncertainties, process and measurement noises along with
	uncertainties in the camera, uncertainty in linear velocity is
	considered, in which it is assumed that the linear velocity of the
	camera, $V_c$ is $80\%$ of its real value. In other words, it is
	assumed that the $V_c$ exact value is not accessible and an
	approximation is taken. Furthermore, noises are considered to keep
	in view the existing vibration in the device in which the camera is
	mounted on. Likewise, the design parameters of the approaches are
	set to be the same as the previous simulation. As it is delineated
	in Fig. \ref{fig5}, the performance of the proposed filter is much
	better than the others. Both variance and mean value of estimation
	error are lower for the designed filter and it is clear that the
	proposed method has gained robust performance due to its switching
	essence.
	
	\begin{figure}[]
		\centering
		\subfloat[]{\includegraphics[width=2in]{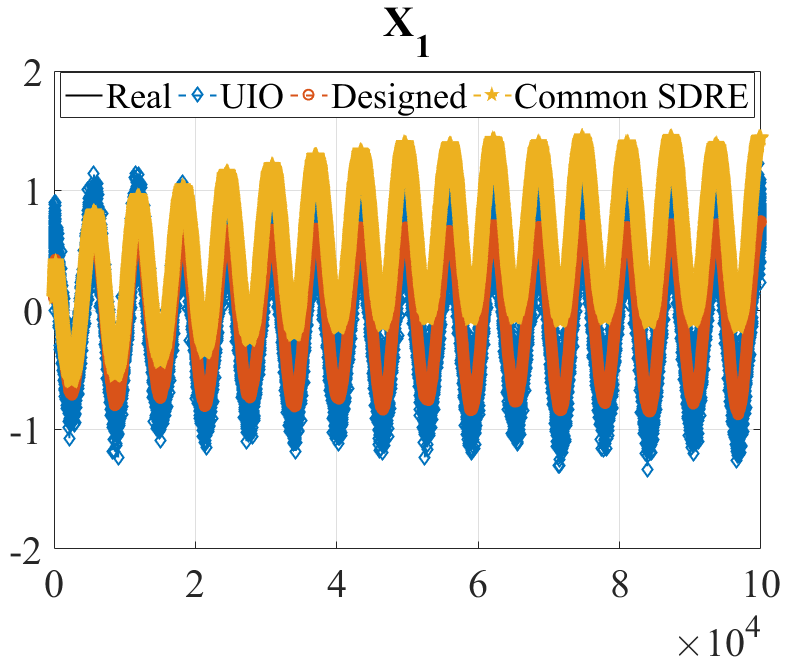}}
		\subfloat[Magnified]{\includegraphics[width=2in]
			{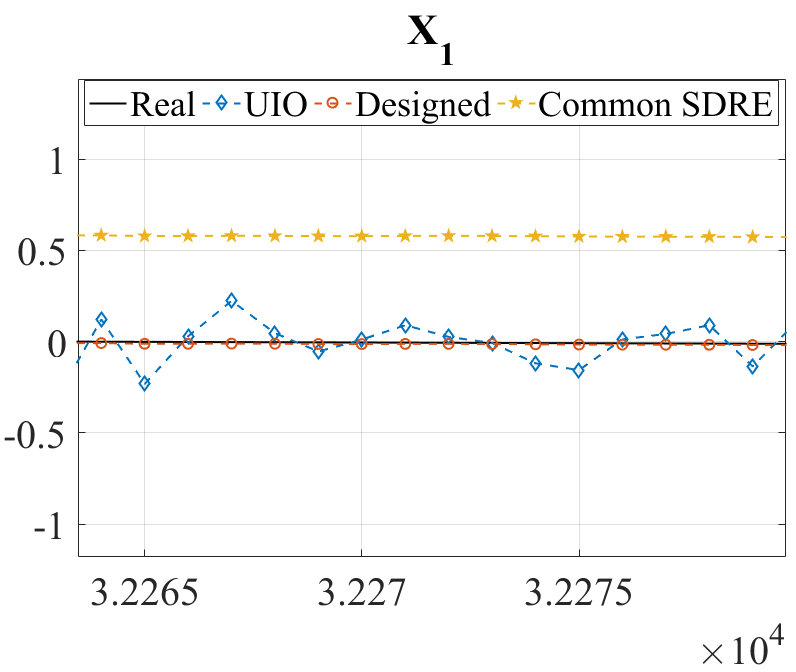}} \\
		\subfloat[]{\includegraphics[width=2in]{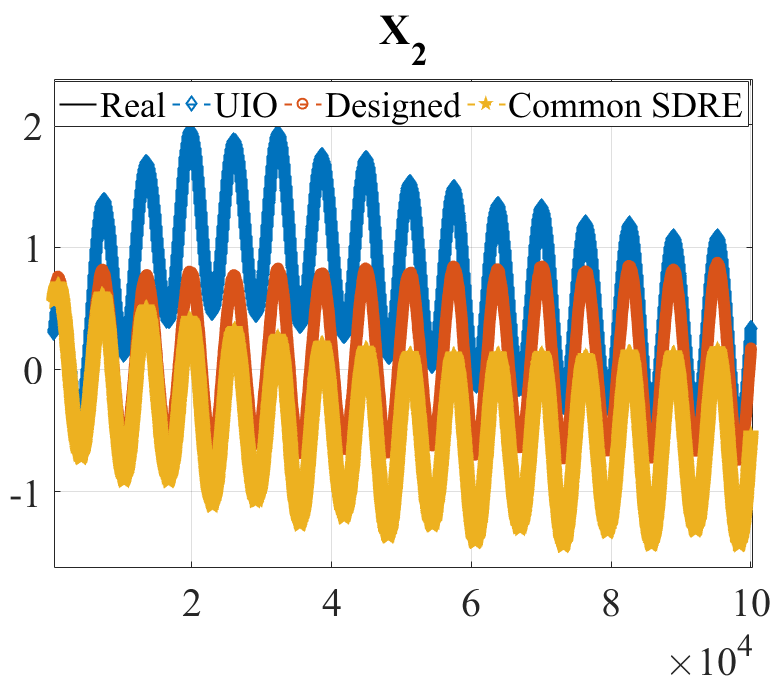}}
		\subfloat[Magnified]{\includegraphics[width=2in]
			{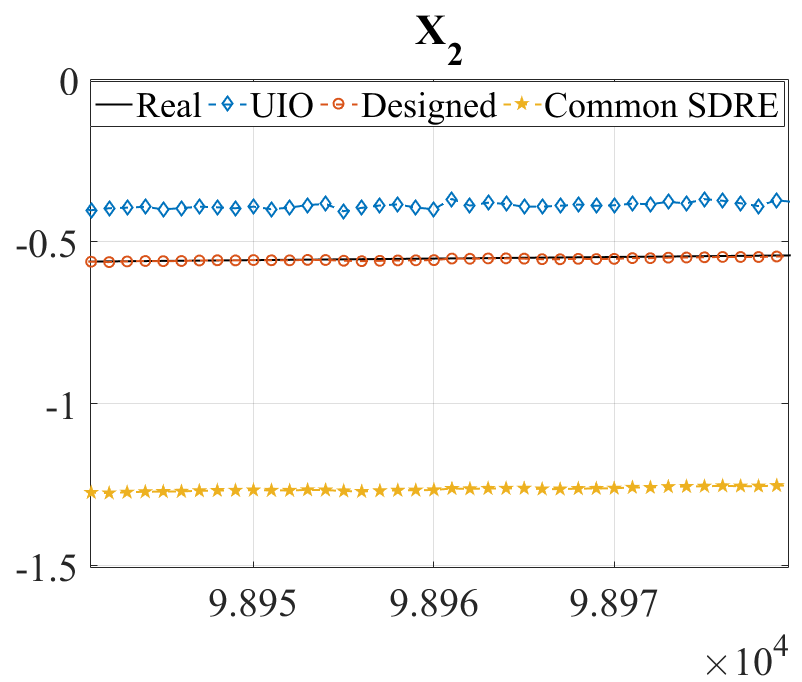}}\\
		\subfloat[]{\includegraphics[width=2.1in]{pics/wo_x3_11_6}}
		\subfloat[Magnified]{\includegraphics[width=2in]
			{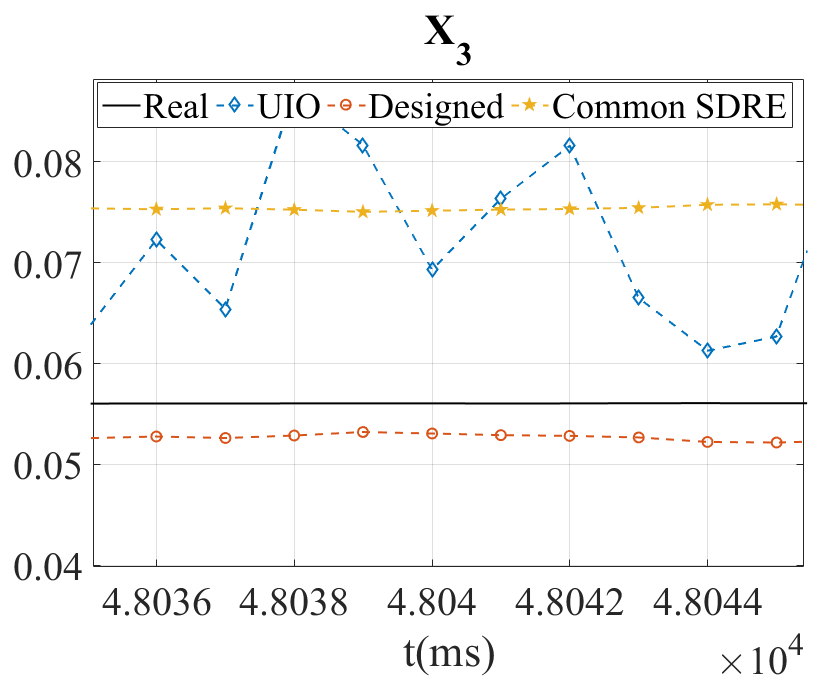}}
		\caption{Estimation results in presence of uncertainties of the model,
			utilizing the same design parameters as in Fig.~\ref{fig4}}
		\label{fig5}
	\end{figure}
	To further investigate the robustness of the suggested approach, a
	Monte-Carlo simulation is performed in which the camera linear
	velocity is assumed to be partially known as an stochastic process
	in a range from $50\%$ to $150\%$ of its real value. To clarify, the
	simulation is accomplished on $1000$ different amounts of $V_c$ in
	the interval of $[0.5V_c,1.5V_c]$ and for each simulation, the
	estimations are obtained through utilizing the three mentioned
	methods for about 10$sec$ with 1$ms$ sampling-rate.
	
	\emph{Remark 14:} To  take $V_c$ as uncertainty, each of its
	elements is considered as a stochastic value independently.
	
	\begin{figure}[]
		\centering
		\subfloat[]{\includegraphics[width=2.2in]{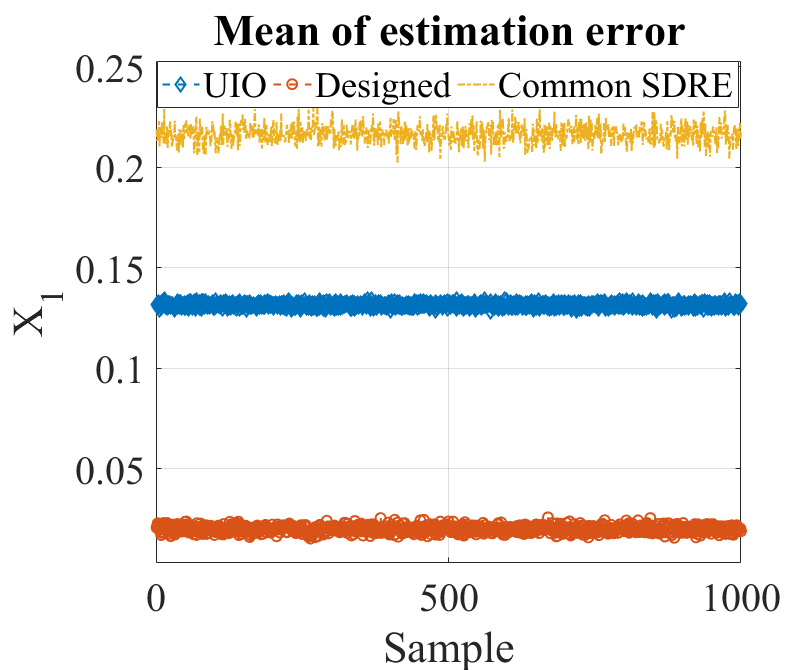}}
		\subfloat[Magnified]{\includegraphics[width=2.2in]
			{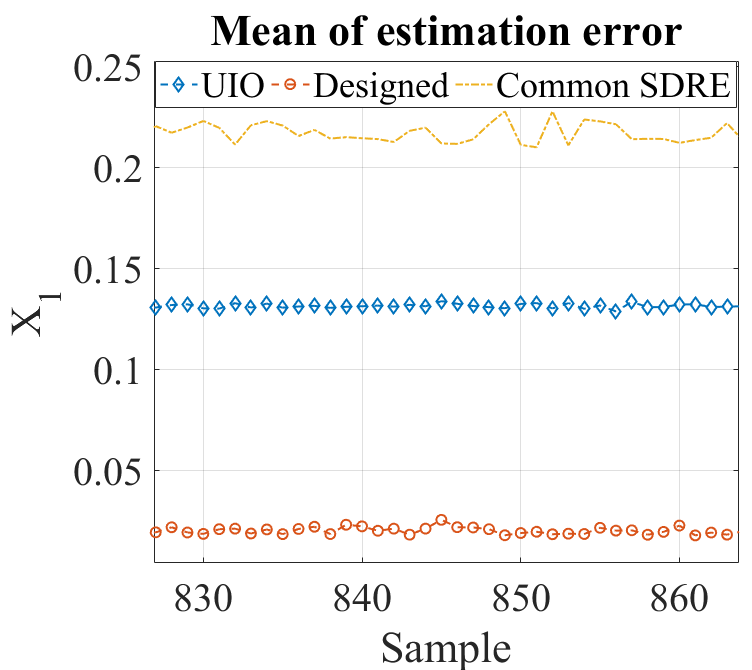}} \\
		\subfloat[]{\includegraphics[width=2.2in]{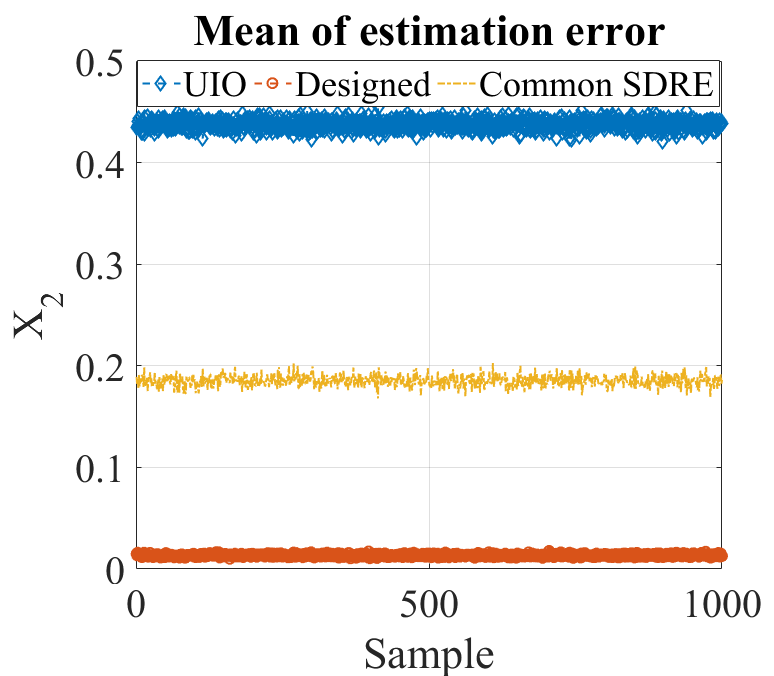}}
		\subfloat[Magnified]{\includegraphics[width=2.2in]
			{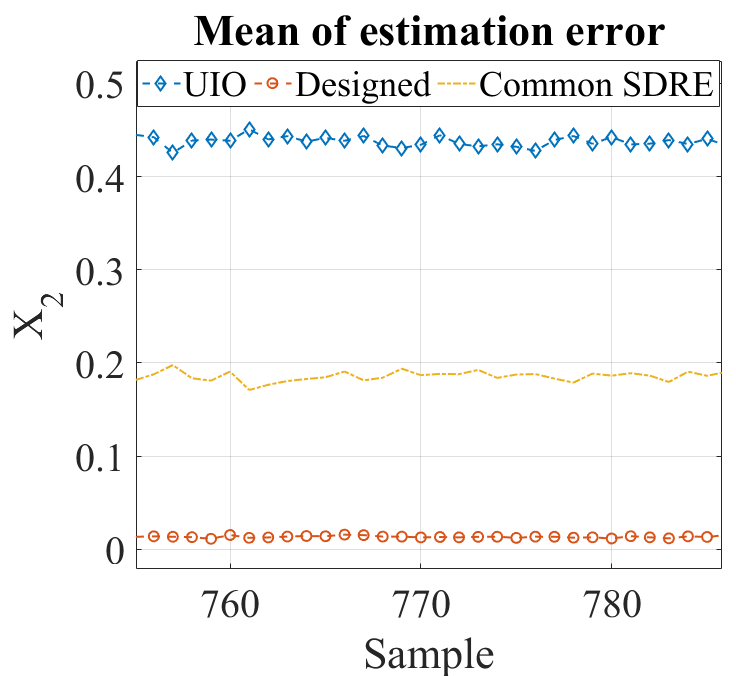}}\\
		\subfloat[]{\includegraphics[width=2.2in]{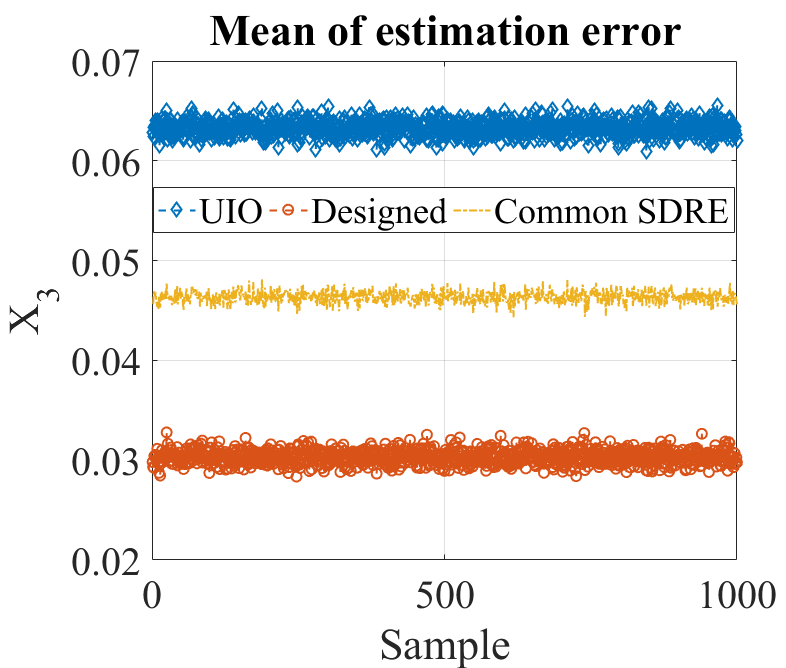}}
		\subfloat[Magnified]{\includegraphics[width=2.2in]
			{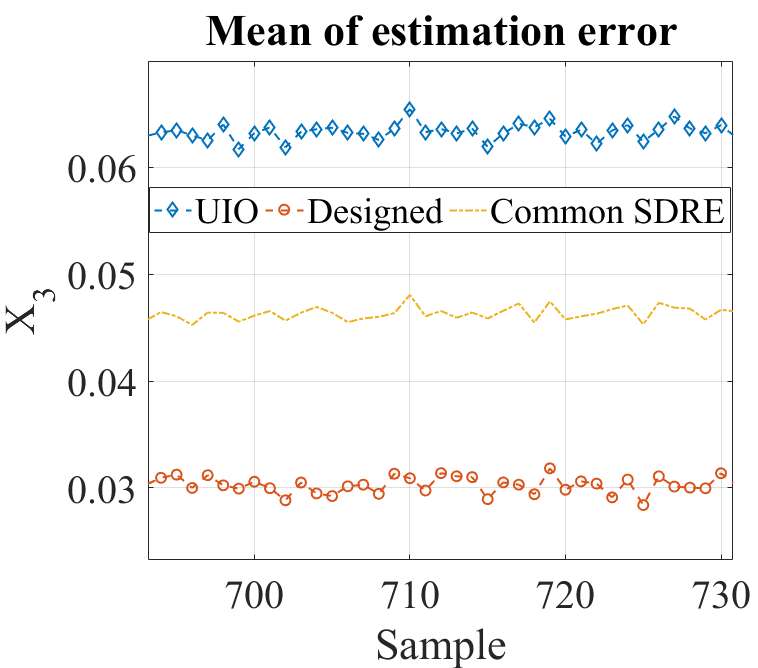}}
		\caption{Mean value of estimation error as a result of Monte-Carlo simulation}
		\label{fig6}
	\end{figure}
	\begin{figure}[]
		\centering
		\subfloat[]{\includegraphics[width=2.2in]{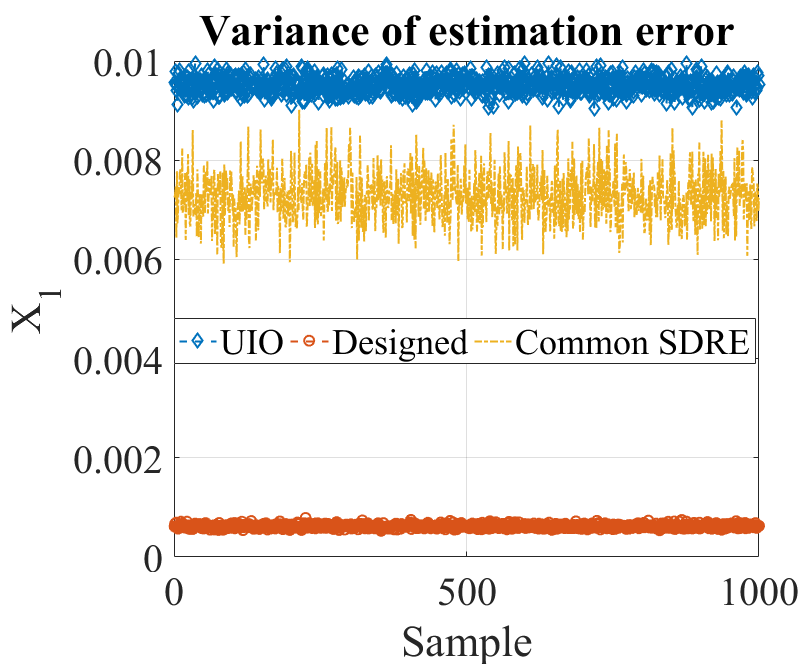}}
		\subfloat[Magnified]{\includegraphics[width=2in]
			{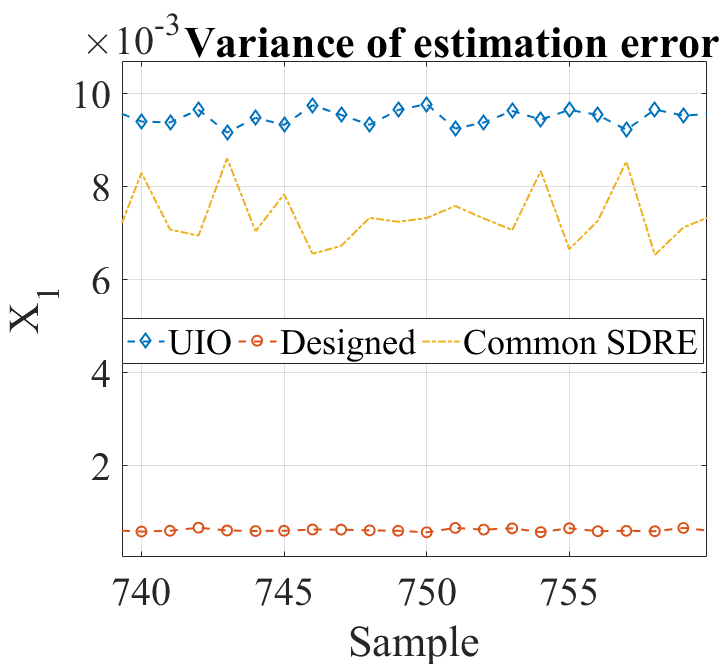}} \\
		\subfloat[]{\includegraphics[width=2.2in]{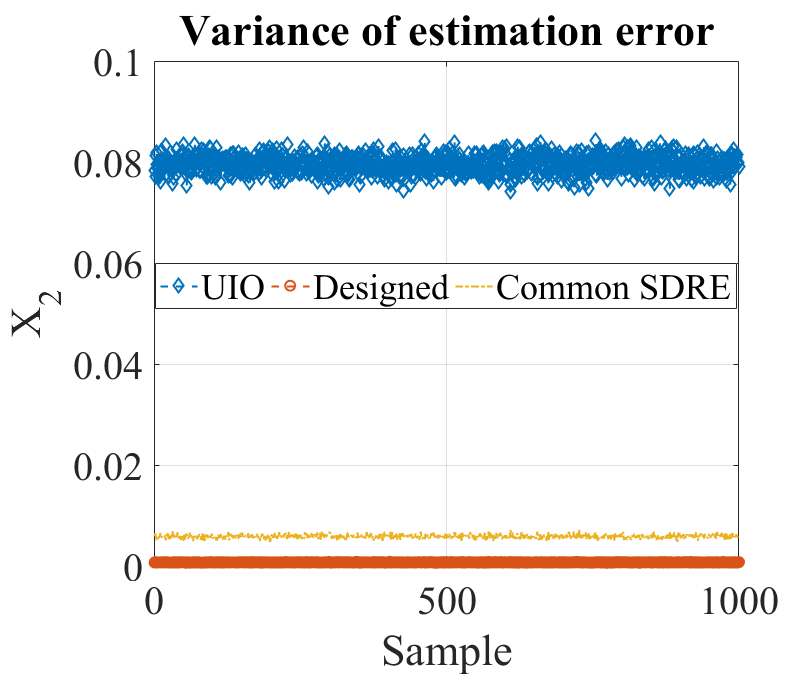}}
		\subfloat[Magnified]{\includegraphics[width=2in]
			{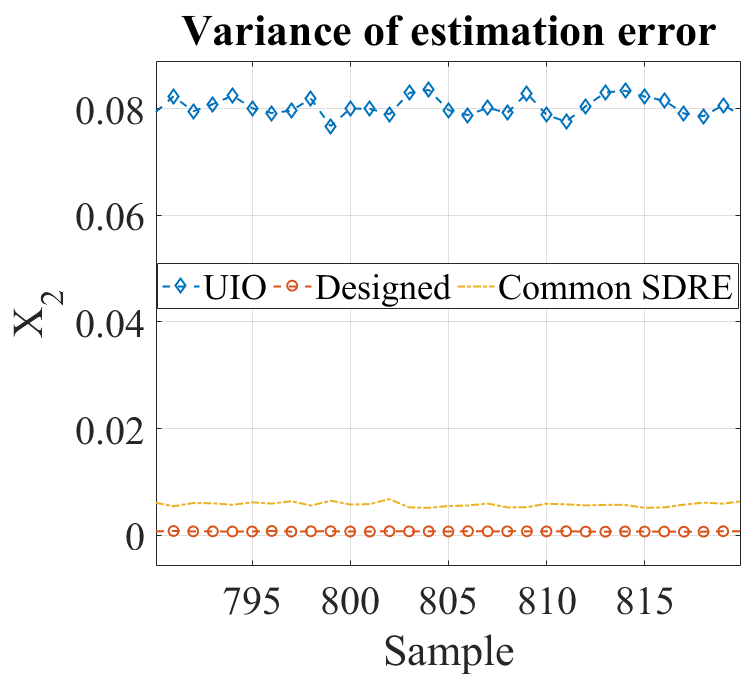}}\\
		\subfloat[]{\includegraphics[width=2.2in]{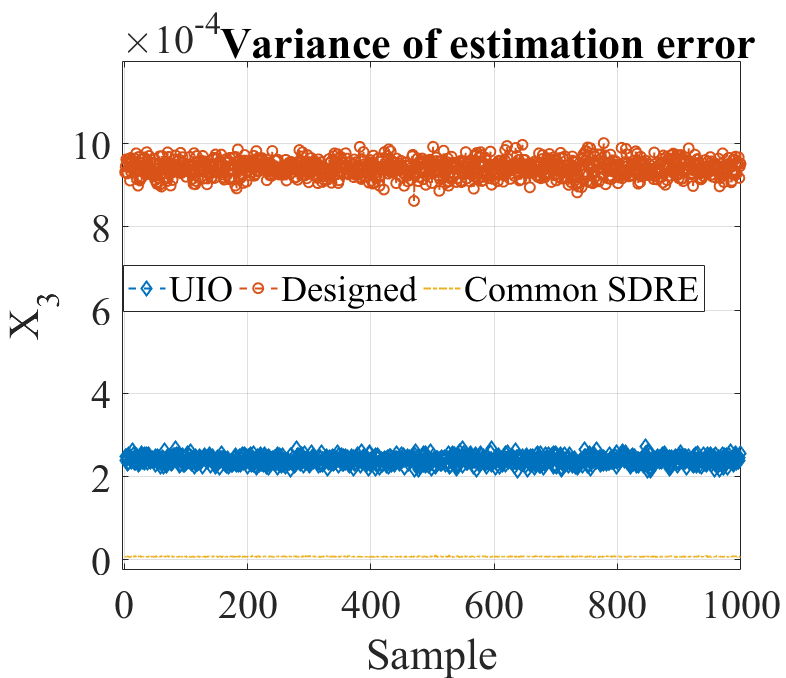}}
		\subfloat[Magnified]{\includegraphics[width=2in]
			{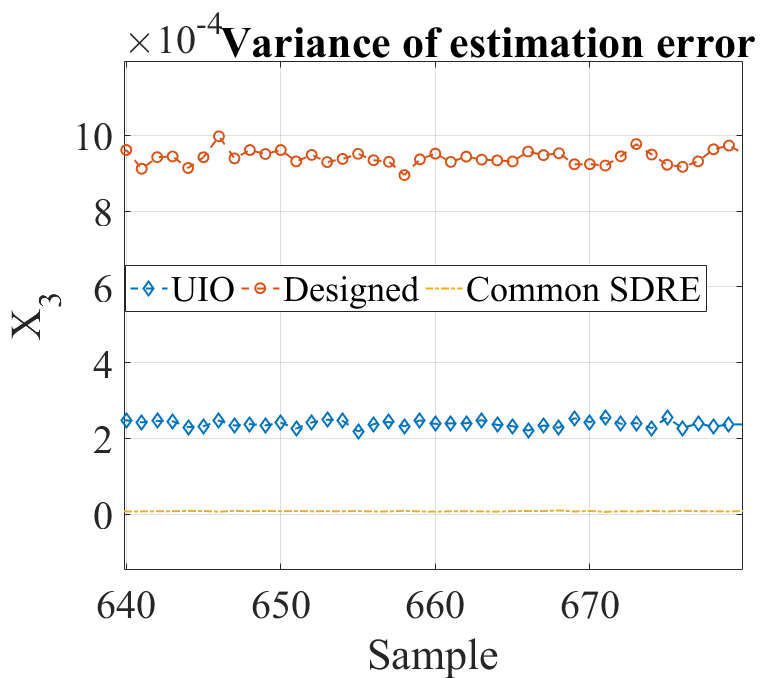}}
		\caption{Variance of the estimation error as a result of Monte-Carlo simulation}
		\label{fig7}
	\end{figure}
	As a result of  the Monte-Carlo simulation, the estimation error
	mean value is shown in Fig. \ref{fig6}. This index is much lower for
	the designed filter in estimating all the three state variables. The
	other important touchstone is the variance of estimation error which
	is given in Fig .\ref{fig7}. It is seen that the estimation error
	variance of the proposed method is much lower than the others in
	other two methods in $x_1$ and $x_2$ estimation. However, this is
	not the case for  the third state $x_3$, but the amount of this
	variance is of order $10^{-4}$ and less important than the other two
	states estimation. Consequently, utilizing switching improves the
	robustness of the common SDRE and yields a robust performance in the
	presence of model uncertainties compared even to UIO filter.
	
	\section{Experimental Results}
	In this section, experimental results are presented in which both
	longitudinal and lateral distances of frontal dynamic objects are
	estimated for a host vehicle. \textcolor{black}{The overall experimental setup is illustrated in Fig. \ref{exp_set}. The camera is located as indicated on the right, while the Jetson Tx2  is seen on the left.}
	\begin{figure}[t]
		\centering
		\includegraphics[width=4in]{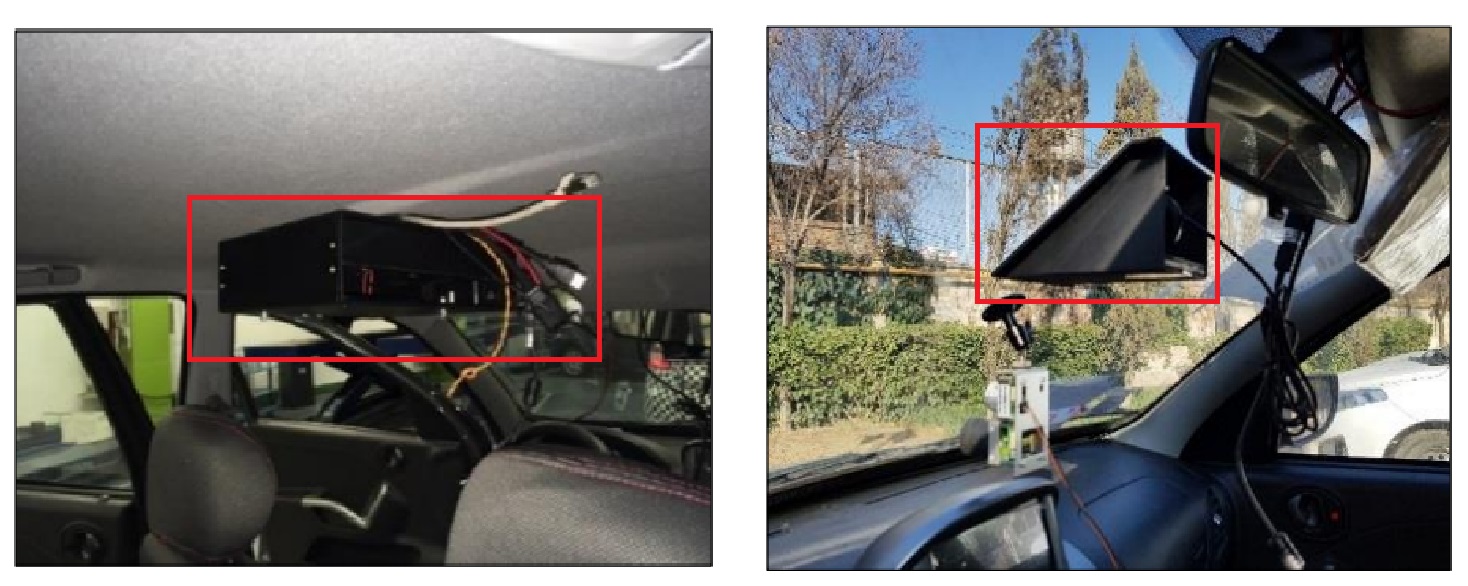}
		\caption{The experimental setup}
		\label{exp_set}
	\end{figure}
	In what follows, first, the concept of employing CNNs
	alongside an OpenCV tracker is presented. Later the mentioned model
	in (\ref{61}) is reduced to a fourth-order state space due to having
	a 2D motion plane. Then the output model developed in (\ref{69}) is
	analyzed for possible employment as an observable system. Finally,
	to verify the proposed algorithm, the radar data is utilized and the
	estimation values are compared to that of the radar measurements. To
	get a proper intuition on the proposed approach, the block diagram
	of the whole system implemented in the experiments is presented in
	Fig. \ref{fig8}. As it is seen in this figure, camera images are
	processed by a robust tracker that contains both YOLO CNN and the
	median flow OpenCV tracker. The outputs of this block are the
	observations that are used in the proposed estimator to estimate
	both longitudinal and lateral distances. In the literature, almost
	all researches reported using a pre-mounted sign on the object in
	performing experiments, although it is a serious limitation for
	these approaches when it comes to the real applications. To tackle
	this problem, in this paper a CNN is utilized to detect all
	potential moving objects through a single image. In addition,
	considering the reported results in~\cite{ourref}, employing CNNs
	beside a Median-flow tracker yields a promising performance in
	tracking objects through image sequences. Consequently, after the
	initial detection and feeding of the tracker, every '$n$' frames the
	YOLO CNN is re-executed to detect new objects and modify the
	tracker, to result in a real-time and robust performance
	implementation. Furthermore, coordinates of the object contact point
	with the ground are utilized as observations and inputs to the
	proposed filter to estimate lateral and longitudinal distances. The
	reduced state space model to be used for the proposed filter is
	presented as:
	\begin{equation}
	\label{70}
	x(k+1) = A(x(k))x(k) + B_1(x(k))V_c(k) + B_2(x(k))\omega(k)
	\end{equation}
	in which $\omega$ and $V_c$ are the angular and linear velocities of
	the camera, respectively, and
	\begin{equation}
	\label{71}
	A(x(k)) = \left[
	\begin{matrix}
	1&-\frac{T}{2}(x_2x_6)_k & \frac{T}{2}(x_3)_k & 0 \\
	0 & 1  &  0 & -\frac{T}{2}(x_3^2)_k \\
	0 & 0 & 1 & 0 \\
	0 & 0 & 0 & 1
	\end{matrix}\right]
	\end{equation}
	\begin{equation}
	\label{72}
	B_{1}(x(k)) = \left[
	\begin{matrix}
	0 & -\frac{T}{2}(x_3)_k & \frac{T}{2}(x_2x_3)_k  \\
	0 & 0 & \frac{T}{2}(x_3^2)_k   \\
	0  & 0 & 0  \\
	0  & 0 & 0
	\end{matrix}\right]
	\end{equation}
	\begin{equation}
	\label{73}
	B_{2}(x(k)) = \left[
	\begin{matrix}
	\frac{T}{2}(1 + x_2^2)_k & -\frac{T}{2}(x_1x_2)_k & -\frac{T}{2}(x_1)_k   \\
	\frac{T}{2}(x_2x_3)_k & -\frac{T}{2}(x_1x_3)_k &0 \\
	0  & 0 & 0 \\
	0  & 0 & 0 \\
	\end{matrix}\right]
	\end{equation}
	while the output model is:
	\begin{equation}
	\label{74}
	y(k) = h(x(k))x(k)
	\end{equation}
	It is noteworthy to mention that here $x(k)$ is equal to
	$[x_2,x_3,x_5,x_6]^T$ where $x_5$ and $x_6$ are $V_{q2}$ and
	$V_{q3}$, respectively. Note that, $x_1$ is taken as a constant
	value in these equations which yields $x_4$ to be zero.
	\begin{figure}[t]
		\centering
		\includegraphics[width=4in]{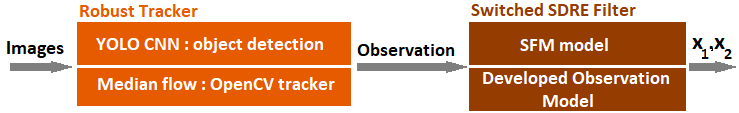}
		\caption{Block diagram of the suggested approach}
		\label{fig8}
	\end{figure}
	
	\emph{Remark 15:} A vital thing for having a practical
	approach is to optimize the framework as much as possible to further
	enhance the performance. Utilizing YOLO at each frame may result in
	better performance. However, the computational cost is far higher
	than that of the proposed algorithm of using YOLO alongside
	medianflow tracker.
	
	Since the camera has some parameters which are used in  (\ref{69}),
	it is imperative to calibrate this observation model before the
	implementation. In this regard, a data set is used which consists of
	both radar and image data. This data set is gathered through a
	scenario in which a car has a sweeping movement. Applying YOLOv3 on
	the images results in the bounding boxes for each image, and this
	may be utilized to attain a function which relates the contact
	points of the objects with the ground, to the frontal distances.
	\begin{figure}[]
		\centering
		\subfloat[]{\includegraphics[width=4in]{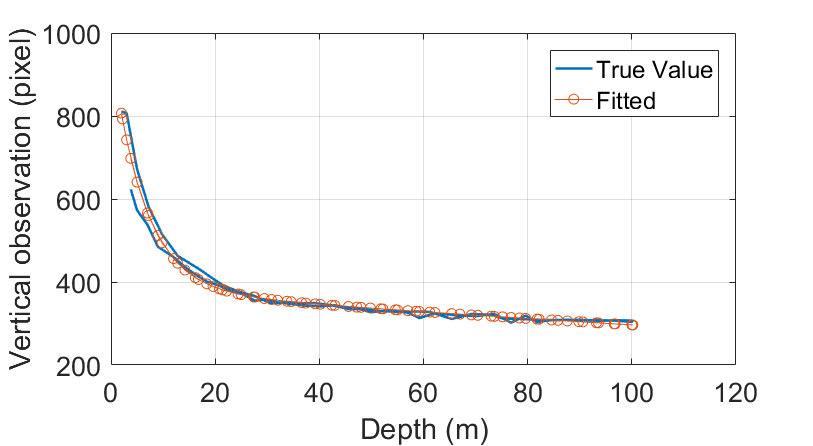}
			\label{fig9_a}}\\
		\subfloat[]{\includegraphics[width=4in]{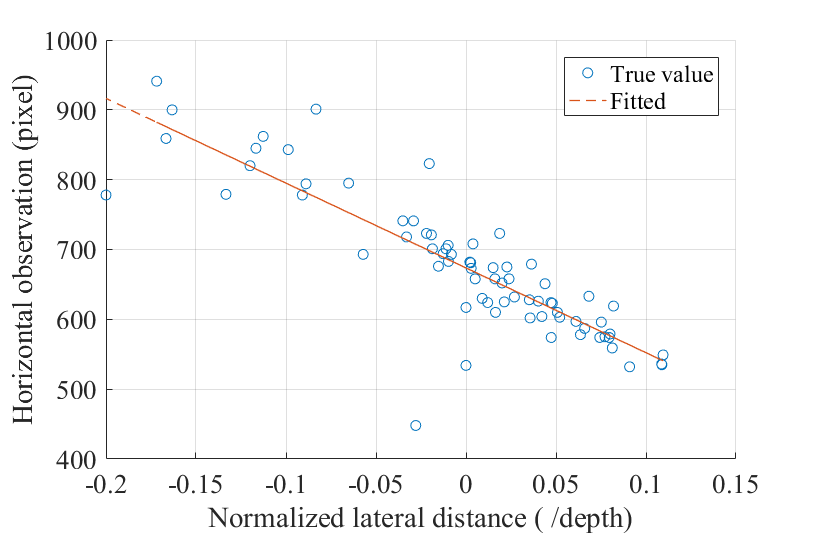}
			\label{fig9_b}}
		\caption{Nonlinear functions derived from plotting vertical and horizontal observations with respect to the true depth and the normalized lateral distance taken from the radar data}
		\label{fig9}
	\end{figure}
	Fig. \ref{fig9_a} indicates the nonlinear behavior of the practical
	vertical observations with respect to the ground truth longitudinal
	distances. As it is seen in this figure, due to the inherent
	nonlinear behavior of these states, a linear model would not be
	sufficient to be used as the observation model. Consequently, an
	alternative is to use multiple linear models considering (\ref{69}),
	or to use a persistent nonlinear model to encapsulate this behavior.
	In this paper we propose using such model for the whole range of
	distances up to $100$ meters. To this end, the following nonlinear
	function is fitted on the ground truth data as shown in Fig.
	\ref{fig9_a}.
	\begin{equation}
	\label{76}
	y_2  = 588 \exp{(-0.161 \frac{1}{x_3})} + 382.6 \exp{(-0.002547 \frac{1}{x_3})}
	\end{equation}
	Using the same approach, Fig. \ref{fig9_b} depicts the relation
	between normalized lateral distance and horizontal observations.
	Since there is no need to use a nonlinear function in here, the
	normalized lateral distance is employed in a linear function to fit
	the empirical observations as follows:
	\begin{equation}
	\label{77}
	y_1  = -1214x_2 + 673.6
	\end{equation}
	Use (\ref{76}), (\ref{77}) into (\ref{74}), to finalize the
	observation model as:
	\begin{equation}
	\label{78}
	h(x(k)) = \left[
	\begin{matrix}
	\frac{y_1}{x_2} & 0& 0& 0  \\
	0 & \frac{y_2}{x_3} & 0& 0
	\end{matrix}\right].
	\end{equation}
	Note that, since $x_2$ goes to zero when both the object and the
	host car are in a same line, it is essential to omit the constant
	term in (\ref{77}), and to consider it as another input to the
	system.
	
	To achieve real-time performance on an Nvidia Jetson Tx2 board, tiny
	YOLO~\cite{ref14} has been used, which yields a limitation in depth
	estimation range. Note that by utilizing YOLOv3, depth estimation is
	valid up to $100$ meters.
	However, in tiny YOLO this range is
	limited to about $40$ meters. In fact, the restriction is applied by
	the CNN structure, while more layers result in the detection of
	farther objects. Albeit both CNNs are employed in this research, the
	results of utilizing YOLOv3 are reported to evaluate the whole
	algorithm for the range of $100$ meters. Moreover, a multi-thread
	framework is designed to extend this algorithm to multi-object
	tracking applications in which there is a main thread responsible
	for producing observations, and for each object a thread is created
	which will be executed for the estimation of the lateral and
	longitudinal distances of that object. By this means the proposed
	method is generalized for any number of the objects in the scene.
	\begin{figure}[]
		\centering
		\subfloat[]{\includegraphics[width=2.5in]{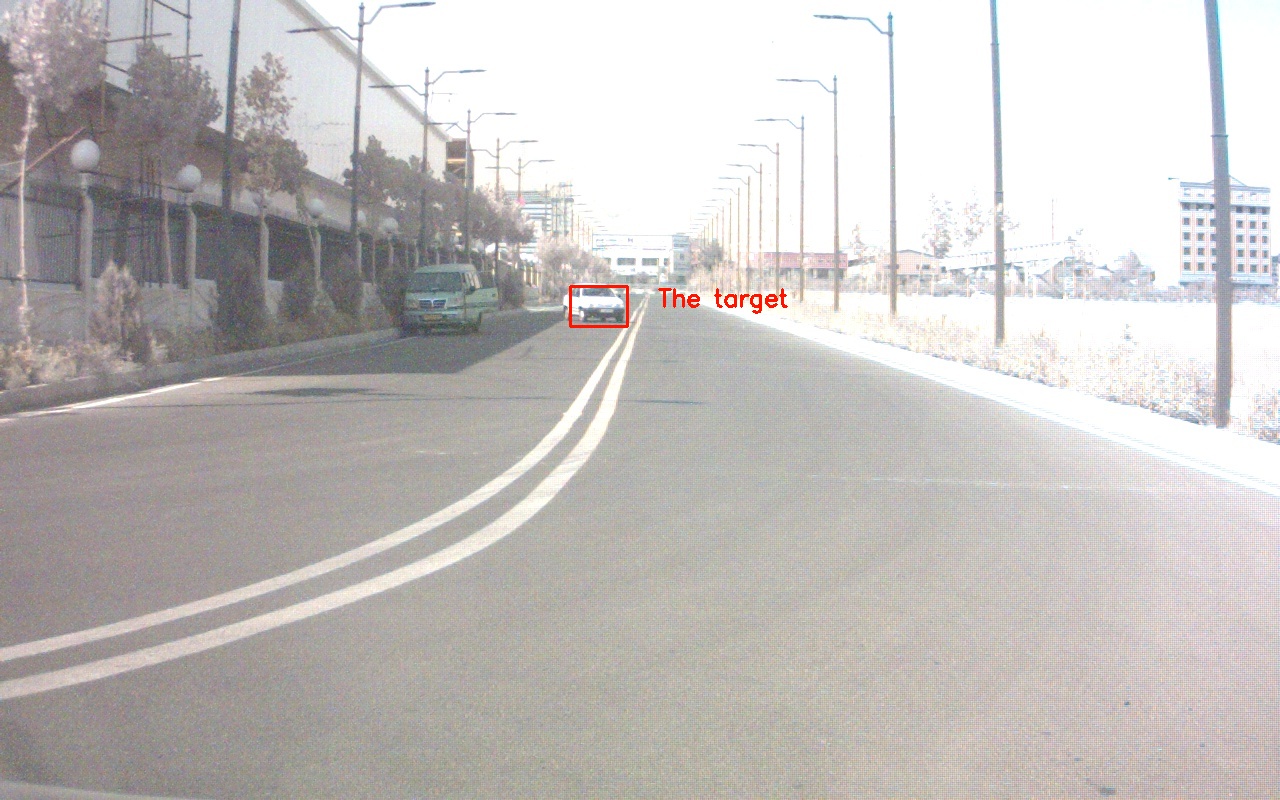}}
		\subfloat[]{\includegraphics[width=2.5in]{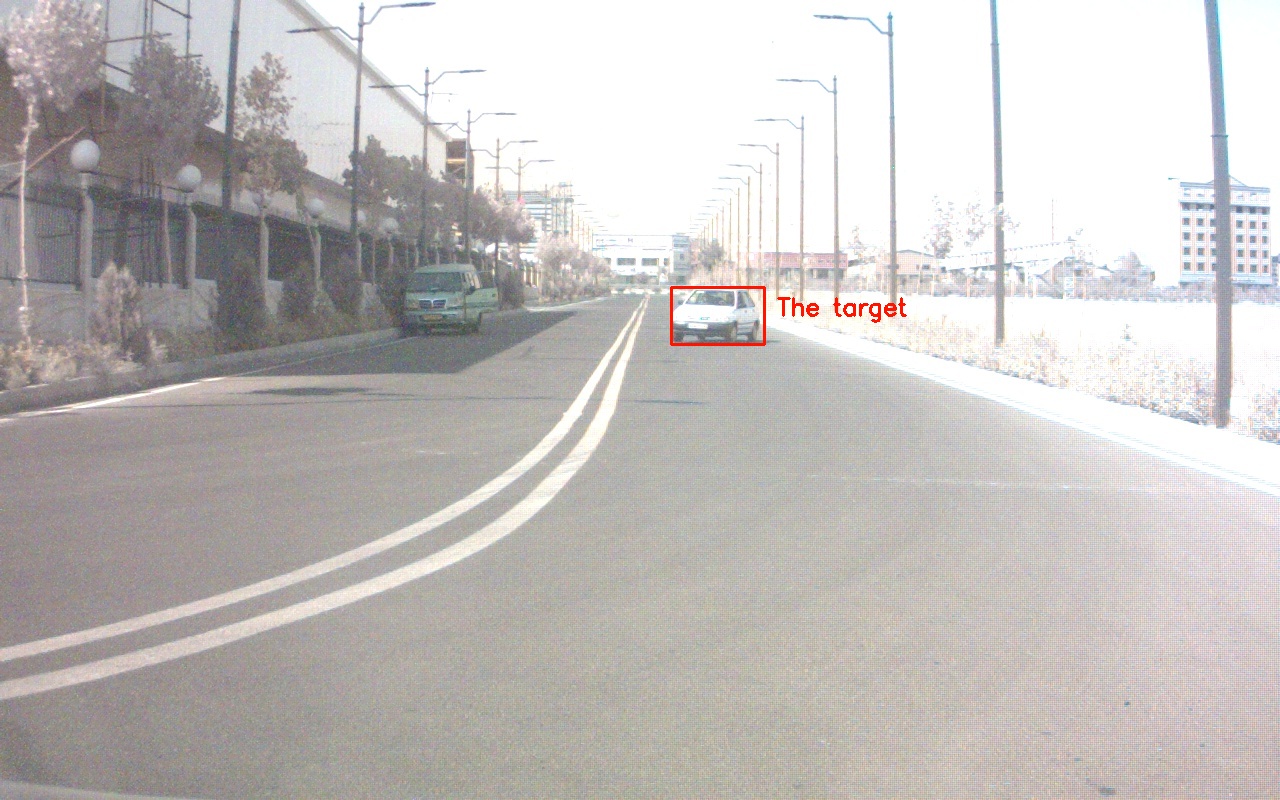}} \\
		\subfloat[]{\includegraphics[width=2.5in]{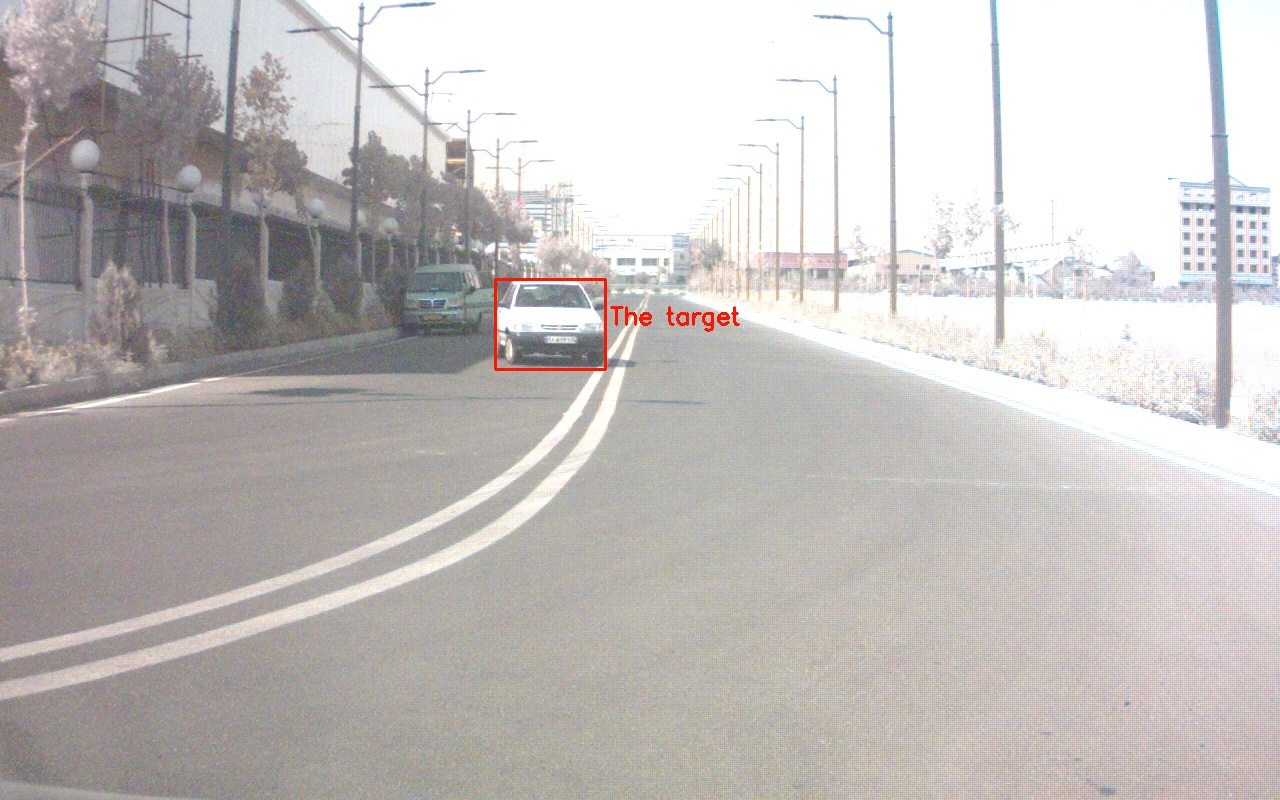}}
		\subfloat[]{\includegraphics[width=2.5in]{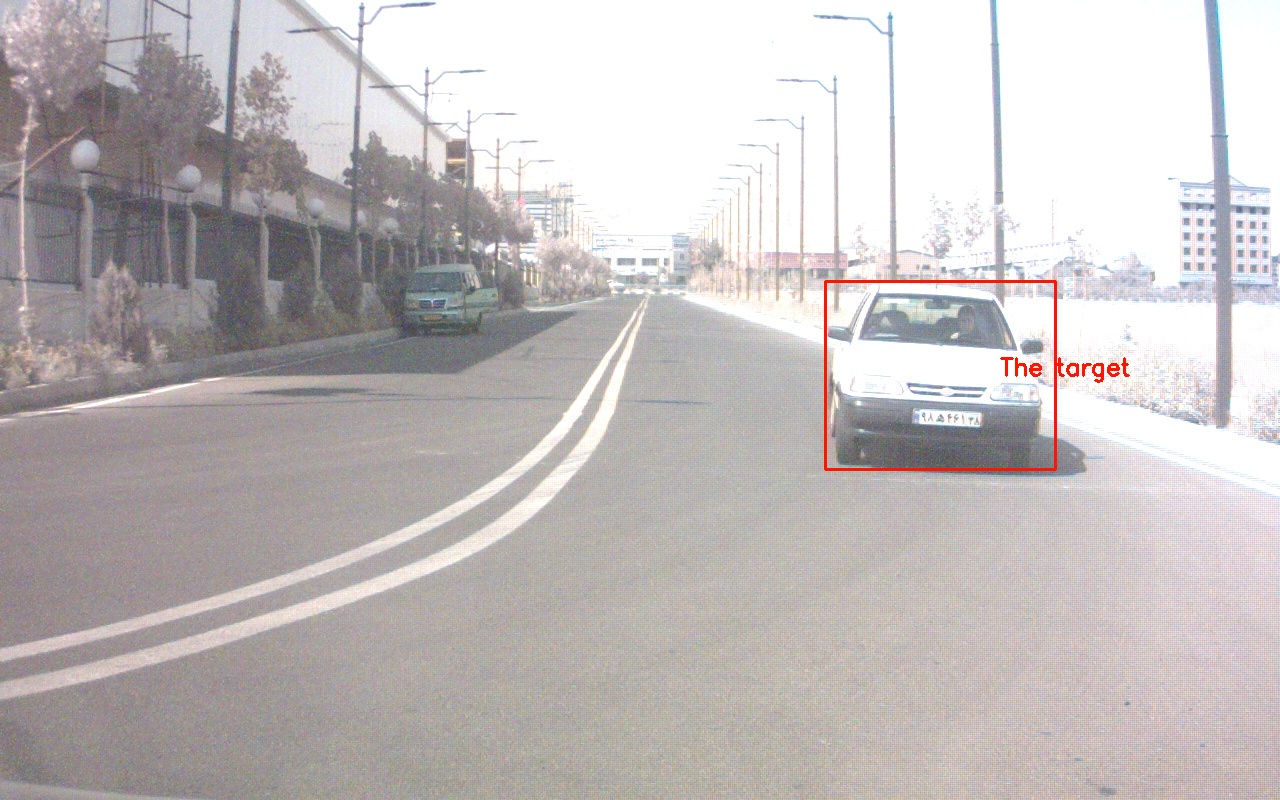}}
		\caption{Sample images of the test scenario}
		\label{fig10}
	\end{figure}
	
	Experiments are conducted on a video sequence utilizing the
	aforementioned model. The video encompasses $34$ frames which $4$ of
	them are reported in Fig. \ref{fig10}. The experiment scenario is
	designed such that a vehicle comes toward the host car with a curly
	maneuver, with the aim to estimate the lateral and longitudinal
	distances of the moving car by using a monocular camera. In fact,
	the scenario represents an exaggerated motion of the frontal car as
	a critical situation in which the driver has lost the vehicle
	control, and it comes to the left lane and in front of the host car.
	\begin{figure}[]
		\centering
		\subfloat[]{\includegraphics[width=4in]{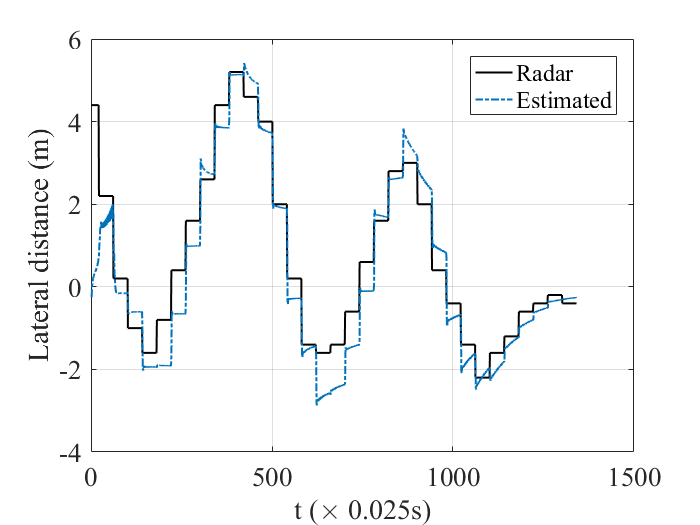}
			\label{fig12_a}}\\
		\subfloat[]{\includegraphics[width=4in]{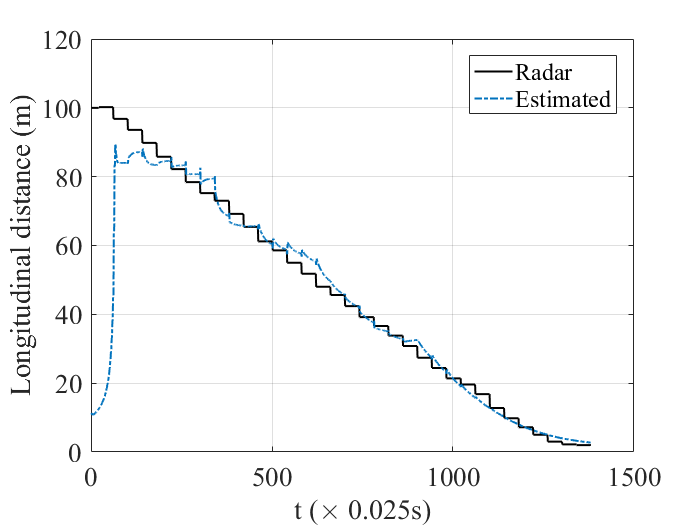}
			\label{fig12_b}}\\
		\caption{Estimated longitudinal and lateral distances }
		\label{fig12}
	\end{figure}
	
	Fig.~\ref{fig12} depicts the proposed results in lateral and
	longitudinal distance estimation, while as mentioned earlier, radar
	data has been used to verify the accuracy of the results. As it can
	be seen, the proposed algorithm suitably determines the movement
	pattern of the object precisely, and furthermore, it estimates the
	absolute value of the lateral distance efficiently. In fact,
	according to this result, the estimation error for the longitudinal
	distance is impressively less than $3\%$ for the range below $50$
	meters, and this increases to  just $5\%$ for the range of $70$
	meters. 
	\textcolor{black}{Compared to the results of the presented machine learning based approaches in the literature \cite{ref_new_4}, the accuracy is at least two times higher while the estimation is performed in a dynamic scene. In fact, this is the first approach presented in the literature which covers the estimation problem of both longitudinal and lateral distances in a tracking scenario, and more importantly, for a practical application.}
	Note that, the camera is not expected to yield very accurate
	results for farther objects, while this causes no problem in
	applications such as obstacle avoidance, since the suggested method
	performs well in the required operation range. 
	\textcolor{black}{Another important point to ponder is, regarding the estimation equations, any error in the longitudinal distance estimation directly effects on the lateral distance one which further demonstrate the applicability of the proposed approach since the results for the lateral distance are promising with almost no drift.} 
	
	\emph{Remark 16:} The main thread which produces  the observations, is much
	slower compared to the other threads. Thus, results in Fig.
	\ref{fig12} are prone to be in a zero-order hold (ZOH) form.
	
	\emph{Remark 17:} The framework has been
	implemented such that for each object in the image, a distinct
	thread is set. Each of these threads consists of an SFM module with
	about $1ms$ computation time. As a result, the proposed algorithm
	performs in real-time, and it is applicable in real-world
	applications. 
	
	\section{Conclusions}
	In this article, a framework is suggested to detect and track
	frontal dynamic objects in an autonomous vehicle motion plane.
	First, a switched SDRE filter is proposed as an effective method to
	solve the general form of the SFM problem in the presence of
	uncertainties, compared to that of commonly used UIO method. To
	analyze the robustness of the proposed approach, a Monte Carlo
	simulation is performed and a comparative study on three filters
	reveals the effectiveness of the proposed method. By utilizing a
	newly developed observation model alongside the SFM, an observable
	model for an autonomous vehicle in motion plane is derived.
	Moreover, to obtain observations from a monocular camera, CNNs are
	employed beside the Median-flow tracker. Reported simulation results
	verify the theoretical development of the proposed filter. As a
	result, the Monte Carlo simulation indicates the superiority of the
	proposed switched SDRE filter among both the commonly used UIO
	approach and the common SDRE. Considering both the mean and the
	variance of the estimation error, the suggested estimator has lower
	indexes in the presence of uncertainties. In order to have a
	real-time implementation in practice, a multi-thread framework is
	implemented. Experimental results show a promising performance in
	frontal distance estimation. Our future research is focused on the
	expansion of the experiments to the cases with intermittent
	observations, in which recurrent neural networks may assist for
	better estimation of the frontal object positions.
	
	\section*{Acknowledgement}
	Authors are thankful to professor Azadi and his team especially Parisa Masnadi from SAIPA automotive company for their supportive help to accomplish the experiments. 
	
	\section*{References}
	
	\bibliography{Refarticle}
	
\end{document}